%% file: main.tex
\documentclass{article}

\usepackage{amsmath, amssymb, amsthm, amsfonts, bbm}
\usepackage{fullpage}
\usepackage{natbib}
\usepackage{algorithm}
\usepackage{algpseudocode}
\usepackage{graphicx}
\usepackage{subcaption}
\usepackage{booktabs}
\usepackage{multirow}

\DeclareMathOperator*{\argmax}{arg\,max}
\DeclareMathOperator*{\argmin}{arg\,min}

\newtheorem{proposition}{Proposition}

\usepackage{xcolor}
\definecolor{DarkGreen}{rgb}{0.1,0.5,0.1}

\begin{document}

\title{Optimal Local Explainer Aggregation for Interpretable Prediction}

\author{Qiaomei Li\footnotemark[1] \and Rachel Cummings\footnotemark[2] \and Yonatan Mintz\footnotemark[1]}

\renewcommand{\thefootnote}{\fnsymbol{footnote}}
\footnotetext[1]{Department of Industrial and Systems Engineering, University of Wisconsin, Madison. Emails: {\tt \{qli449, ymintz\}@wisc.edu}. Most of this work was completed while Q.L. and Y.M. were at the Georgia Institute of Technology.}
\footnotetext[2]{H. Milton Stewart School of Industrial and Systems Engineering, Georgia Institute of Technology. Email: {\tt rachelc@gatech.edu}. Supported in part by NSF grants CNS-1850187 and CNS-1942772 (CAREER),  a Mozilla Research Grant, a Google Research Fellowship, and a JPMorgan Chase Faculty Research Award. Part of this work was completed while R.C. was a Google Research Fellow at the Simons Institute for the Theory of Computing.}

\renewcommand{\thefootnote}{\arabic{footnote}}

\maketitle


%
%


\begin{abstract}
A key challenge for decision makers when incorporating black box machine learned models into practice is being able to understand the predictions provided by these models. One set of methods proposed to address this challenge is that of training surrogate \emph{explainer models} which approximate how the more complex model is computing its predictions. Explainer methods are generally classified as either \emph{local} or \emph{global} explainers depending on what portion of the data space they are purported to explain. The improved coverage of global explainers usually comes at the expense of explainer \emph{fidelity} (i.e., how well the explainer's predictions match that of the black box model). One way of trading off the advantages of both approaches is to aggregate several local explainers into a single explainer model with improved coverage. However, the problem of aggregating these local explainers is computationally challenging, and existing methods only use heuristics to form these aggregations.
  
In this paper, we propose a local explainer aggregation method which selects local explainers using non-convex optimization. In contrast to other heuristic methods, we use an integer optimization framework to combine local explainers into a near-global aggregate explainer.  Our framework allows a decision-maker to directly tradeoff coverage and fidelity of the resulting aggregation through the parameters of the optimization problem. We also propose a novel local explainer algorithm based on information filtering. We evaluate our algorithmic framework on two healthcare datasets: the Parkinson's Progression Marker Initiative (PPMI) data set and a geriatric mobility dataset from the UCI machine learning repository. Our choice of these healthcare-related datasets is motivated by the anticipated need for explainable precision medicine. We find that our method outperforms existing local explainer aggregation methods in terms of both fidelity and coverage of classification. It also improves on fidelity over existing global explainer methods, particularly in multi-class settings, where state-of-the-art methods achieve 70\% and ours achieves 90\%.
\end{abstract}

\input{intro.tex}

\input{methodology.tex}
\input{local.tex}

\input{experiments.tex}

\input{conc.tex}

\bibliography{pdinterp}
\bibliographystyle{abbrvnat}

\newpage


\appendix

\input{appcluster}

\input{applocal}

\input{applocalexp}

\input{appfffs}

\input{appfigs}

\end{document}

%% file: intro.tex
\section{Introduction}

When applying machine learning and AI models in high risk and sensitive settings, one of the biggest challenges for decision makers is to rationalize the insights provided by the model. In applications such as precision medicine, both accuracy of prediction (e.g., anticipated efficacy of treatment) and transparency of how predictions are made are key for obtaining informed consent. However, the types of models that typically achieve the highest levels of accuracy also tend to be extremely complex, and even machine learning experts describe them as ``black boxes'' because it is difficult to explain why certain predictions are made \citep{breiman2001statistical}. One popular approach to resolve this trade off between explainability and accuracy is to extract simple \emph{explainer} models from complex black box models.  These models are intended to provide a simplified facsimile of the true model that is more useful for human interpretation of the generated predictions. 


Two important metrics for evaluating explainer models are \emph{fidelity} and \emph{coverage}. Fidelity measures how well the explainer's predictions match the predictions of the original black box model, and coverage measure the fraction of the data universe that is reasonably explained by the explainer model.  Explainer methods are generally classified as either \emph{global} or \emph{local}, based on how they trade off between these two quantities.  Global explainers attempt to explain the full black box model across the entirety of the data.  These models have a hard constraint to provide 100\% coverage, often at the expense of fidelity.  Local explainers, on the other hand, sacrifice coverage to potentially provide higher fidelity explanations in a smaller region of the data, usually centered around one single prediction.



Recent proposals suggest finding a middle ground between these two extremes by forming global (or near-global) explainers by aggregating local explainer models \cite{ribeiro2016should}.  This approach would allow the decision-maker to trade off among coverage, fidelity, and explainability: including more local explainers in the aggregate model would improve coverage and fidelity, at the cost of a more complex---and hence less interpretable---aggregate model. However, the problem of computing the best subset of local explainers to explain a given black box model is combinatorial in nature, and hence computationally challenging to solve. All existing methods for building aggregate explainers use only heuristic approaches, and thus do not provide theoretical performance guarantees. 

In this work, we present a novel way of constructing provably optimal aggregate explainer models from local explainers. We use an integer programming (IP) based optimization framework that trades off between coverage of the aggregate model and fidelity of the local explainers that comprise the aggregate model.  We also propose a novel local explainer methodology that uses feature selection as an information filter, and is designed for effective use in aggregation. We empirically validate the performance of this framework in two data-driven healthcare applications: Parkinson's Disease progression and geriatric mobility. These experimental results show that our model provides higher fidelity than existing methods.

\subsection{Related Work}
\label{sec:lit_rev}
Our paper builds on previous work in the broader field of interpretable machine learning. The two primary types of interpretable learning include models that are interpretable by design \citep{aswani2019behavioral}, and black box models that can be explained using global explainer \citep{wang2015falling,lakkaraju2016interpretable,ustun2016supersparse,bastani2018interpreting} or local explainer \citep{ribeiro2016should,ribeiro2018anchors} methods. 

Models that are interpretable by design are perhaps the gold standard for interpretable ML. However, these models often require significant domain knowledge to formulate and train, and are therefore not suited for exploratory tasks such as the precision healthcare applications we study in Section \ref{s.globalexp}. 

Global explainer methodology attempts to train an explainable model (such as a decision tree with minimal branching) to match the predictions of a black box model across the entirety of its feature space. While these models may provide some understanding on the general behavior of the black box model, if the relationship between features and black-box predictions is too complex, then the global explainer may remove many subtleties that are vital for validation and explanation. 

Local explainer methods attempt to train simpler models centered around the prediction for a single data point. The most commonly used local explainer methods are Local Interpretable Model-Agnostic Explanations (LIME) \citep{ribeiro2016should} and anchors \citep{ribeiro2018anchors}. While local methods cannot validate the full black box model, they are useful for understanding the subtleties and justification for particular predictions. In recent literature several other local explainer methods have been proposed that draw inspiration from this stream \citep{rajapaksha2019lormika,sokol2020limetree, plumb2018model}. 

A third option which has been explored in recent literature, is that of aggregating several local explainer models to form a near-global explainer. Generally speaking, these methods have a budget for the maximum number of local explainers that can be incorporated into the aggregation and attempt to maximize possible coverage and fidelity within this budget. One method proposed to form such aggregate explainers is the submodular pick method \citep{ribeiro2016should}, which computes feature importance scores and greedily selects the features with highest importance. \cite{van2019global} argue the Submodular Pick Algorithm has its limitations on predicting global behaviors from local explainers, and that the choice of aggregation function for local explainers is important for performance. They introduce the Global Aggregations of Local Explanations (GALE) method, which can be used to analyze how well the aggregation explains the model's global behavior. They compared the performance of global LIME aggregation with other global aggregation methods for binary and multi-class classification tasks, and found that different aggregation approaches performed best in binary and multi-class settings.

The methodology we propose in this paper builds on top of these existing explainer aggregation methods. In contrast to existing approaches which are heuristic in nature, we formulate the problem of choosing local explainers for the aggregate as an optimization problem. By doing so, our methods can produce explainer aggregates that provide both higher fidelity and higher coverage than existing approaches. In addition, our formulation includes parameters that allows for a direct tradeoff between coverage, fidelity, and interpretability. We believe this approach is especially appropriate for problems in explainable precision healthcare, where the relationship between diagnostic screening measures and the diagnosis is quite complex, and the model should incorporate the richness of this relationship in its predictions.


In addition we propose a local explainer approach in Section \ref{s.local} that includes a feature selection subroutine to improve explainability.  Prior work on feature selection includes instance-wise feature selection \citep{chen2018learning} and Instance-wise Variable Selection using Neural Networks (INVASE) \citep{yoon2018invase}. These approaches select the important features for each sample point using networks for classification with and without the features. Shapley values have also been used for complex model predictions, such as Shapley Sampling Values \citep{vstrumbelj2014explaining} and Shapley Additive Explanations (SHAP) \citep{lundberg2017unified}, which computes Shapley values and presents the explanation as an additive feature attribution method. In contrast to these methods, our feature selection approach relies on a mutual information filter \citep{brown2012conditional} to identify important features. While mutual information has been used in the past for feature selection, we introduce a computationally efficient way to compute this mutual information for the specific use of training local explainer models.



\subsection{Our Contributions}


In this paper, we formulate the problem of aggregating local explainers into an aggregate explainer algorithm as a non-convex optimization problem. In particular, we show that this aggregation problem can be written as an integer program (IP), that can be solved effectively using commercial solvers. This formulation is also helpful as it allows us to directly tradeoff coverage and fidelity of the resulting aggregation through the parameters of the optimization problem. 

Additionally, we design a new methodology for training local explainers for effective use in aggregation.  Our local explainer algorithm directly computes locally significant features using an information filter.  We introduce a novel computationally efficient algorithm for this filtering step, and our approach results in simpler (i.e., more interpretable) local explainers compared to prior work that used regularization for feature selection. 

To validate our results, we compare our optimization based methodology against four other state of the art methods on two real world data sets. Both data sets come from the applications in the healthcare space. The first uses the Parkinson's Progression Marker Initiative (PPMI) \cite{PPMI}, where we create explainer methods for a model tasked with screening patients for Parkinson's Disease. The second uses a dataset of Geriatric activity, where we explain the predictions of a model that classifies the physical activity of geriatric patients to prevent falling. Our experiments show that our optimization method outperforms many of the existing explainer methods in terms of fidelity and coverage. In particular, when we examine cases of explaining multi-class model predictions, our explainer method can achieves 90\% fidelity at 40-50\% coverage, while existing global methods only achieved 70\% fidelity, albeit at 100\% coverage. Our results show that our approach on the Pareto frontier of the fidelity and coverage tradeoff. Our IP framework outperforms existing aggregation methods in terms of  both coverage and fidelity across all potential aggregation budgets (i.e., numbers of local explainers in the aggregate model). 

\paragraph{Organization.} Section \ref{s.global} details our aggregate explainer framework and its formalization as an IP.  Section \ref{s.local} summarizes our local explainer methodology and feature selection algorithm. Section \ref{s.globalexp} provides empirical validation and compare the performance of our IP-based approach with other local and global explainer methods. Preliminaries are contained within each section, and additional technical details can be found in the appendices.




%% file: methodology.tex
\section{Explainer Aggregation Methodology}\label{s.global}
%

Explainer models which can generalize to a large portion of the feature space are critical for model diagnostics and transparency. However, an explainer that is constrained to explain the space feature space is likely have low fidelity since, by design, the explainer model is less complex than the black box model it is purported to explain. However, simpler models can achieve higher fidelity by attempting to explain the local behavior of the black box model at the cost of lower coverage. 

One way to address the tradeoff between coverage and fidelity is to create a near-global aggregate explainer model by combining several local explainer models. Existing approaches have used this idea \citep{ribeiro2016should} by formulating the construction of an aggregate explainer as an optimization problem: maximize coverage of the explainer subject to a constraint on the total number of local explainers included in the aggregate. Solving this optimization problem is conjectured to be computationally intractable \citep{ribeiro2016should}, and prior work has only attempted to solve it using heuristics.

In this section, we formulate the problem of constructing the aggregate explainer explicitly as an integer linear program that can be solved efficiently using commercial solvers, and allows us to directly trade off coverage and fidelity.

\subsection{Mathematical Programming Formulation of Aggregation Problem}
\label{sec:math_prog_form}

To formulate the optimization problem of constructing the aggregate explainer, we must first formally define the concepts of coverage and fidelity. 

Let $\mathcal{X} \subset \mathbb{R}^m$ be the feature space that is modeled with a black box function, and let $f:\mathcal{X} \rightarrow \mathbb{Z}_+$ be the black box function of interest. Let $\mathcal{L} \subseteq \mathbb{Z}_+$ be the label space in the image of $f$.  We consider our explanation task over a dataset $\mathcal{D}$ containing $n$ ordered pairs $(x_i,y_i)$ for $i \in [n]$, where $x_i \in \mathcal{X}$ are the features values and $y_i \in \mathcal{L}$ is the class label which has been generated by $f$. That is, $y_i = f(x_i)$.


Let $g_{i,r}: \mathcal{X} \rightarrow \mathcal{L}$ denote a local explainer model that explains the local behavior of the black box function $f$ on inputs within a ball of radius $r \in \mathbb{R}_+$ centered around the point $x_i \in \mathcal{X}$. We use $\mathcal{X}_{i,r}:=\{x \in \mathcal{X}: \|x - x_i\| \leq r \}$ to denote the region explained by $g_{i,r}$.

Define an aggregate explainer $\gamma$ to be a set of local explainers centered around a subset of points in $\mathcal{D}$, where the local explainer for point $x_i \in \mathcal{D}$ has radius $r_i$.\footnote{More generally, any local explainers can be aggregated into $\gamma$.  However, we assume the the explainer algorithm only has access to points in $\mathcal{D}$, so we restrict ourselves to only considering these points. It is assumed that the radii $r_i$ are parameters of the problem and hence known to decision-maker.} We will refer to a generic local explainer $g\in \gamma$ and corresponding region of explanation $\mathcal{X}_g$. 
 
Using these quantities we define the \emph{coverage of aggregate explainer $\gamma$} on the data set $\mathcal{D}$ as the total number of points in the data set that are covered by the explanation radius of at least one explainer contained in $\gamma$. We denote the coverage as:  
\begin{equation}\label{eq.cov}
\text{Cov}(\gamma,\mathcal{D}) = \sum_{x \in \mathcal{D}} \max_{i\in \{i: g_{i,r} \in \gamma \}} \mathbbm{1}[x \in \mathcal{X}_{i,r}].
\end{equation}

Next we note that the fidelity of a single local explainer can be defined as the accuracy of that explainer using the predicted labels of the black box model as the ground truth. We will define the \emph{fidelity of aggregate explainer $\gamma$} on the data set $\mathcal{D}$ as the minimum of the fidelity obtained by each individual local explainer in $\gamma$. We first need to define $\mathcal{D}_g$ as the number of points in the dataset contained in the explanation region of $g$, i.e., $\mathcal{D}_g = \{x \in \mathcal{D}: x\in \mathcal{X}_g\}$. We denote the fidelity of $\gamma$ as: 
\begin{equation}\label{eq.fid}
\text{Fid}(\gamma,\mathcal{D}) = \min_{g \in \gamma} \frac{1}{|\mathcal{D}_{g}| }\sum_{x \in \mathcal{D}} \mathbbm{1}[g(x) = f(x)]\mathbbm{1}[x \in \mathcal{X}_{g}].   
\end{equation}

While one could instead define the fidelity of $\gamma$ as the average of the fidelities of its component explainers, our choice to use the minimum fidelity gives a stricter measure of how well the aggregate explainer captures the behavior of the black box model. This stricter measure is more appropriate for the healthcare applications we consider in Section \ref{s.globalexp}, where a minimum standard of care is required. Note also that while we may be interested in the coverage and fidelity of $\gamma$ across the entirety of $\mathcal{X}$, computing these quantities may be intractable or impossible in practice when $\mathcal{X}$ is not known a priori. Thus we consider these quantities only across an $r$-ball covering of our dataset.

Let $K$ denote the budget of the maximum number of local explainers that can be contained in $\gamma$, and let $\varphi$ be the minimum fidelity required for the aggregate explainer. Then the problem of computing an aggregate explainer can be formulated as the following optimization problem:
\begin{equation}
\label{eq:global_problem}
\begin{aligned}
\max_{\gamma}&\ \left\{\text{Cov}(\gamma,\mathcal{D}) : \text{Fid}(\gamma,\mathcal{D}) \geq \varphi, |\gamma| \leq K\right\}.
\end{aligned}
\end{equation}

\subsection{Reformulation as Integer Program (IP)}
\label{sec:reform}
As written, optimization problem \eqref{eq:global_problem} is not trivial to solve, and in particular could require enumerating all possible subsets $\gamma$ of local explainers.  To address this challenge, we propose reformulating problem \eqref{eq:global_problem} as an Integer Program (IP) that can be solved using current commercial software. To do this, we first define three sets of binary variables that we will call $w_i,y_j,z_{ij}$. Let $w_i$ be a binary variable that is equal to 1 if explainer $g_{i,r_i} \in \gamma$.  That is, $w_i = \mathbbm{1}[g_{i,r_i} \in \gamma]$. Let $y_j$ be a binary variable that is equal to 1 if point $j$ is covered by the aggregate explainer $\gamma$. That is $y_j = \mathbbm{1}[x_j \in \cup_{g \in \gamma} \mathcal{X}_{g}]$.  Finally, let $z_{ij}$ be a binary variable that is equal to 1 if explainer $g_{i,r_i} \in \gamma$ covers point $x_j$. That is, $z_{ij} = \mathbbm{1}[x_j \in \mathcal{X}_{i,r_i}]$. We can now define the coverage and fidelity of aggregate explainer $\gamma$ as integer programs written in terms of these three sets of variables.

\begin{proposition}
	\label{prop:cov}
	$\text{Cov}(\gamma,\mathcal{D})$, the coverage of aggregate explainer $\gamma$ on dataset $\mathcal{D}$, can be expressed with the following set of integer variables and constraints:
	\begin{equation}
	\begin{aligned}
	\text{Cov}(\gamma,\mathcal{D}) &=\sum_{j=1}^n y_{j}, \\
	\text{s.t.} \quad     &z_{ij} \leq w_i, \quad i,j \in [n],\\
	& y_j \geq z_{ij}, \quad i,j \in [n],\\
	& y_j \leq \sum_{i= 1}^n z_{ij}, \quad  j\in [n],\\
	&\|x_i - x_j\|z_{ij} \leq r_i,   \quad i,j \in [n].
	\end{aligned}
	\end{equation}
\end{proposition}
\begin{proof}
	Recall the definition of $\text{Cov}(\gamma,\mathcal{D})$ as given in Equation \eqref{eq.cov}. We will directly reconstruct this definition using the binary variables defined above. First note that through a simple direct substitution we obtain $\text{Cov}(\gamma,\mathcal{D}) = \sum_{x \in \mathcal{D}} \max_{i\in \{i: g_{i,r} \in \gamma \}} z_{ij}$.  Since taking the maximum of binary variables is equivalent to the Boolean OR operator, we see that $y_j = \max_{i\in \{i: g_{i,r} \in \gamma \}} z_{ij}$, which provides us with the first equality. The next two inequalities directly capture that a local explainer $g_{i,r_i}$ can only explain point $x_j$ if $g_{i,r_i}$ is included in $\gamma$, which is a standard way of modeling conditional logic in IP \citep{wolsey1999integer}.  The next two constraints come from modeling the Boolean OR operator using integer constraints \citep{wolsey1999integer}. The final constraint ensures that a point $x_j$ can only be covered by an explainer $g_{i,r_i}$ if $x_j \in \mathcal{X}_{i,r_i}$, thus preserving the local region of the local explainer.
\end{proof}

Next we consider the minimum fidelity constraint.

\begin{proposition}
	\label{prop:fide}
	The constraint $\text{Fid}(\gamma,\mathcal{D}) \geq \varphi$ can be modeled using the following set of integer linear constraints:
	\begin{equation}
	\begin{aligned}
	&\|x_i - x_j\|z_{ij} \leq r_i,   \quad i,j \in [n],\\
	&z_{ij} \leq w_i, \quad i,j \in [n],\\
	&\sum_{j =1}^n\big( \mathbbm{1}_{\{f(x_j)=g_{i,r_i}(x_j)\}} - \varphi \big)z_{ij} \geq  0, &i\in [n].
	\end{aligned}
	\end{equation}
\end{proposition}
\begin{proof}
	The first two constraints ensure proper local behavior of the local explainer as in Proposition \ref{prop:cov}. Thus we will focus the derivation of the final constraint. Using the definition of $\text{Fid}(\gamma,\mathcal{D})$ in Equation \eqref{eq.fid} and directly substituting variables for indicators, we can express the lower bound constraint as, 
	\[\min_{\{i: g_{i,r_i} \in \gamma\}} \frac{1}{|\mathcal{D}_{g_{i,r_i}}| }\sum_{x_j \in \mathcal{D}} \mathbbm{1}[g_{i,r_i}(x_j) = f(x_j)]z_{ij} \geq \varphi.\]
	Note that if the minimum over all explainers $g_{i,r_i}$ must have fidelity of at least $\varphi$, then every local explainer must also have fidelity at least $\varphi$. This allows us to disaggregate this constraint across all $i \in [n]$. Let us consider the constraint for a single local explainer $g_{i,r_i} \in \gamma$. Using the definition of $z_{ij}$, we note that $|\mathcal{X}_{i,r_i}| = \sum_{x_j \in \mathcal{X}} z_{ij} $. Thus the new lower bound fidelity constraint for a single explainer can be written as:
	\begin{equation}
	\frac{\sum_{j =1}^n \mathbbm{1}[g_{i,r_i}(x_j) = f(x_j)]z_{ij}}{\sum_{j =1}^n z_{ij} } \geq \varphi.
	\end{equation}
	Note that the denominator of the left hand side can only be zero when the numerator is also zero because $\sum_{j =1}^n z_{ij} \geq \sum_{j \in \mathcal{X}} \mathbbm{1}[g_{i,r_i}(x_j) = f(x_j)]z_{ij}$. This means that we can multiply both sides of the inequality by the sum $\sum_{j =1}^n z_{ij}$ while still maintining its validity. Distributing $\varphi$ and combining like terms gives us with the form of the constraint presented in the proposition statement.
\end{proof}

We can then use these expressions to for coverage and fidelity to re-write our optimization problem as an integer program that can then be solved using commercial solvers.

\begin{proposition}
	\label{prop:reform}
	The optimization problem in \eqref{eq:global_problem},
	\begin{align*}
	\max_{\gamma}&\ \{\text{Cov}(\gamma,\mathcal{D}) : \text{Fid}(\gamma,\mathcal{D}) \geq \varphi, |\gamma| \leq K\},
	\end{align*}
	 can be written as the following integer program:
	\begin{equation}
	\begin{aligned}
	\max &\sum_{j = 1}^n y_{j}, \\
	\text{s.t.}  \quad   &z_{ij} \leq w_i, \quad i,j \in [n],\\
	& y_j \ge z_{ij},  \quad i,j \in [n], \\
	& y_j \leq \sum_{i\in \mathcal{X}} z_{ij}, \quad  j\in [n],\\
	&\|x_i - x_j\|z_{ij} \leq r_i,   \quad i,j \in [n], \\
	&\sum_{j =1}^n\big( \mathbbm{1}_{\{f(x_j)=g_{i,r_i}(x_j)\}} - \varphi \big)z_{ij} \geq  0, \quad i\in [n], \\
	& \sum_{i \in \mathcal{X}} w_i \leq K, \\
	& y_j,w_i,z_{ij} \in \{0,1\} \quad i,j \in [n].
	\end{aligned}
	\end{equation}
\end{proposition}
\begin{proof}
	The objective function and first five constraints come directly from Propositions \ref{prop:cov} and \ref{prop:fide}. The next constraint comes using the definition of $w_i$ and direct substitution to obtain that $|\gamma| =  \sum_{i \in =1}^n w_i$, which is then used to rewrite the budget constraint from \eqref{eq:global_problem}. The final constraint ensures that our new variables are binary integers.
\end{proof}

%% file: local.tex
\section{Aggregate-Designed Efficient Local Explainer}\label{s.local}

While our main contribution in this paper is the local explainer aggregation methodology, we have additionally designed a new methodology for training local explainers for effective use in aggregation. The key to our methodology is ensuring that local explainers only focus on the most relevant features in the particular region they are designed to explain. In contrast to previous methods that proposed the use of regularization to achieve this goal, we propose directly computing locally significant features using an information filter. Computing such filters are generally computationally expensive and requires the use of numerical integration; however, we introduce an efficient algorithm for filtering out less significant features. This methodology allows us to train local explainers that are significantly less complex than those that use regularization, with better fidelity for their specified region.  In this section we present an overview of our methodology and highlight key results. Further details on the technical specifics of this methodology are deferred to the appendix.

\subsection{Local Explainer Preliminaries}

Building on top of our existing notation from Section \ref{s.global}, let $\Phi= \{1,...,m\}$ be the index set of the features used in the black box prediction. This set can be partitioned into two sets $\Phi_c,\Phi_b \subseteq \Phi$ that respectively represent the set of continuous and binary features. Define the set-valued function $\Phi^*: \mathcal{X} \rightarrow \Phi$ as the function which extracts the minimum set of necessary features to accurately predict the class of a point $x \in \mathcal{D}$.

Formally, let $x[\varphi]$ is an indexing operation that maintains the values of $x$ but only for the features in $\varphi$, and $p$ is the conditional probability mass function of the labels $y$ given the observation of some features.  Then we can write: 
\[ \Phi^*(x) = \argmin_{\varphi \subseteq \Phi}\left\{ |\varphi| : p(y|x) = p(y|x[\varphi])\right\}. \]
If a feature index is not included $\Phi^*(x)$, then it is not required to explain the label of $x$. Note that this problem may be computationally intractable in general since the conditional distributions under $p$ are not known a priori. 

We propose approximating the solution to this problem using \emph{mutual information}, which is an information theoretic quantity that measures the level of correlation between two random variables. If $X,Y$ are random variables with joint density $p$ and marginal densities $p_x,p_y$, then the mutual information between $X$ and $Y$ is denoted $I(X;Y)$ and calculated as:
\begin{equation}
I(X;Y) = \mathbb{E}\left[\log\frac{p(X,Y)}{p_x(X),p_y(Y)}\right]. 
\end{equation}
If $X$ and $Y$ are independent, then $I(X;Y) = 0$; otherwise $I(X;Y) > 0$, meaning that $X$ contains some information about $Y$. A similar quantity can be computed using a conditional distribution given another random variable $Z$, known as the \emph{conditional mutual information} and denoted $I(X;Y|Z)$. 

Our local explainer will use mutual information to estimate which features are relevant for predicting labels. An advantage of using $I(X;Y)$ is that it can be computed using existing methods. However, since it is defined using expectation, this often requires the use of numerical integration. One of the contributions of our algorithm is providing an efficient way to compute this integral using tree traversal.

Finally, we will use $\mathcal{B}(x,r,d)$ to denote the ball around a point $x$ of radius $r$ with respect to a metric $d$. 

\subsection{Local Explainer Overview and Training Procedure}

Our local explainer training methodology is formally presented in Algorithm \ref{alg:loc_exp}. We give a brief overview of its operations here, and defer full details to Appendix \ref{sec:loc_exp}. The algorithm takes in hyper-parameters including number of points $N$ to be sampled for training the explainer, a distance metric $d$, and a radius $r$ around the point $\bar{x}$ being explained. First the algorithm samples $N$ points uniformly from within a $r$ radius of $\bar{x}$; we call this set of points $T(\bar{x})$. Depending on the distance metric being used this can often be done quite efficiently, especially if the features are binary valued or an $\ell^p$ metric is used \citep{barthe2005probabilistic}.  Then using the sampled points, the algorithm uses the Fast Forward Feature Selection (FFFS) algorithm as a subroutine (discussed in Section \ref{s.shortfffs} and formally presented in Appendices \ref{app.fffs}), which uses a mutual-information-based information filter to remove unnecessary features and reduce the complexity of the explainer model. The FFFS algorithm uses an estimate of the joint empirical distribution of $(T(\bar{x}),f(T(\bar{x}))$ to select the most important features for explaining the model's predictions in the given neighborhood using tree traversal. We denote this set of features $\hat{\Phi}$. Then, using these features and the selected points, the local explainer model $g$ is trained by minimizing an appropriate loss function that attempts to match its predictions to those of the black box model.  In principle, a regularization term can be added to the training loss of explainer $g$. However, in our empirical experiments (presented in Appendix \ref{sec:exp_res}), we found that FFFS typically selected at most five features, so even the unregularized models where not overly complex.

\begin{algorithm}
	\begin{algorithmic}[1]
		\caption{Local Explainer Training Algorithm}
		\label{alg:loc_exp}
		\Require sampling radius $r$, number of sample points $N$, black box model $f$, data point to be explained $\bar{x}$, and loss function $L$ for the explainer model $(\bar{x},\bar{y})$
		\State Initialize $T(\bar{x}) = \emptyset$
		\For {$j = \{1,...,N\}$}
		\State Sample $x \sim U(\mathcal{B}(\bar{x},r,d))$
		\State  $T(\bar{x}) \leftarrow T(\bar{x}) \cup x$
		\EndFor
		\State Obtain $\hat{\Phi}(\bar{x}) = \text{FFFS}(T(\bar{x}), \Phi, f)$
		\State Train $g = \argmin_{\hat{g} \in \mathcal{G}}\{\sum_{x\in T(\bar{x})}L(f(x) - \hat{g}(x[\hat{\Phi}])) \}$
		\State \Return g
	\end{algorithmic}
\end{algorithm}

\subsection{Detailed Discussion on Fast Feature Selection}\label{s.shortfffs}

A key step in our algorithm is the use of a mutual information filter to reduce the number of features that will be included in the training of the local explainer. Mutual information filters are commonly used in various signal processing and machine learning applications to assist in feature selection \citep{brown2012conditional}. However, these filters can be quite challenging to compute depending on the structure of the joint density function of the features and labels, and can require the use of (computationally expensive) numerical integration. We handle this challenge by considering an approximation of the density function, using histograms to calculate continuous features. When multiple combinations of features need to be considered as in our setting, the problem of finding the maximum-information minimum-sized feature set is known to be computationally infeasible \citep{brown2012conditional}. As such, our proposed method for computing the filter includes a common heuristic known as \emph{forward selection}, which essentially chooses the next best feature to be included in the selected feature pool in a greedy manner.  Using this method alone would still require recomputing the conditional distribution of the data based on previously selected features, which can result in long run times for large $N$. However, using some prepossessing techniques, we show that these quantities can be stored efficiently using a tree structure, which allows quick computation of the filter.

The general idea of the FFFS algorithm is to consider the feature selection process as a tree construction. Part of this construction relies on an estimate of the empirical density of the features as a histogram with at most $B$ bins and  a preprocessed summary tensor $M \in \{0,1\}^{B\times |\Phi| \times N}$ that indicates which bin of the histogram a feature value for a particular data point lays in. For each entry, $M[b,\varphi,x] = 1$ if the value of feature $\varphi$ at point $x$ falls in the bin $b$. Otherwise, $M[b,\varphi,x] = 0$.  The depth of the tree represents the number of selected features and each node of the tree is a subset of $T(\bar{x})$. 

For example, at the beginning of the selection process, we have a tree with exactly one node $R$ where $R=T(\bar{x})$.  If a binary feature $\varphi_1$ is selected in the first round, then two nodes $a,b$ are added under $R$, where $a = \{x_j: M[0,\varphi_1,j]=1\}$ and $b = \{x_j: M(1,\varphi_1,j)=1\}$. In the second round, the algorithm would use the partition sets $a,b$ to compute the mutual information instead of the complete set $R$. The set $a$ would be used for computing $\hat{p}(\varphi|\varphi_1=0),\; \hat{p}(y|\varphi_1=0) \text{, and }\hat{p}(\varphi;y|\varphi_1=0)$, while $b$ would be used for computing the same quantities conditioned on $\varphi_1=1$. In each round, the leaves $\mathcal{L}$ of the current tree represent the set of partition sets corresponding to all random permutation of selected features information. Therefore, $\mathcal{L}$ provides us sufficient information for calculating the desired mutual information, and the algorithm only outputs the leaves $\mathcal{L}$, not the entire tree. The main algorithmic challenge is to efficiently calculate the marginal distributions $(\hat{p}(\varphi|S), \hat{p}(y|S)$ and joint distribution $\hat{p}(\varphi;y|S)$, which we are able to do using the tree structure.

The detailed structure of the FFFS algorithm used to compute the filtered feature set $\hat{\Phi}$ requires several subroutines, and the formal algorithmic construction for computing the filter is presented across Algorithms \ref{alg:fffs}, \ref{alg:Recur}, \ref{alg:sf}, and \ref{alg:bin} in Appendix \ref{app.fffs}. The main FFFS algorithm is Algorithm \ref{alg:fffs}, and it calls the subroutines for recursion (Algorithm \ref{alg:Recur}), selecting features (Algorithm \ref{alg:sf}), and partitions (Algorithms \ref{alg:bin}). Formal presentation of these algorithms, as well as detailed descriptions, are given in Appendix \ref{app.fffs}.

%% file: experiments.tex
\section{Experimental Results}\label{s.globalexp}

In this section we measure the empirical performance of our explainer aggregation methodology against existing global explainer and aggregation methods. For our experiments we compare the performance of our integer programming method for aggregating local explainers against five state-of-the-art global explainer methods. We consider two local explainer aggregation methods---Submodular Pick and Anchor Points \citep{ribeiro2016should,ribeiro2018anchors}---and three global explainer methodologies---interpretable decision sets \citep{lakkaraju2016interpretable}, active learning decision trees \citep{bastani2018interpreting}, and naive decision tree global explainers \citep{friedman2001elements}. 

We compare these methods in both coverage and fidelity across two different datasets. These datasets are the Parkinson's Progression Marker Initiative (PPMI) \citep{PPMI} data set, where we generate explainers for a black box model aimed at predicting Parkinson's Disease (PD) progression subtypes, and a Geriatric activity data set \citep{torres2013sensor} where we generate explainers for a model that classifies the movement activities of geriatric patients based on wearable sensor data. One important feature of both these datasets is that they enable multi-class classification. The experimental results of this section show that our proposed optimization framework is better suited to these multi-class settings than existing state-of-the-art methods. 

In addition to measuring the performance of our local aggregation methodology on different data sets and classification tasks, we also compare the performance of our information-filter-based decision-tree local explainer and LIME \citep{ribeiro2016should} in the aggregation framework.  We also measure performance for each of the aggregation-based methods under varying budgets of component local explainers. This budget is an informal measure of simplicity and interpretability, where aggregating fewer local explainers leads to a more interpretable aggregate explainer, but may sacrifice fidelity and/or coverage. Our results show that our methodology outperforms existing techniques in terms of fidelity and coverage, especially in the multi-class case.

\subsection{PD Progression Cluster Classification}
For our first set of experiments we used the PPMI data set to classify the disease progression of different patients into several subtypes based on screening measures. The PPMI study was a long run observational clinical study designed to verify progression markers for PD. To achieve this aim, the study collected data from multiple sites and includes lab test data, imaging data, and genetic data, among other potentially relevant features for tracking PD progression. The study includes measurements of all these features for the participants across 8 years at regularly scheduled follow up appointments. The complete data set contains information on 779 patients, and included 548 patients diagnosed with PD or some other kind of Parkinsonism and 231 healthy individuals as a control group. For our analysis we will focus on the first seven visits of this study which correspond to a span of approximately 21 study months, since these visits were conducted relatively close together temporally.

The classification task considered was the disease progression of the patients, and we performed a cluster analysis to generate labels, detailed in Appendix \ref{sec:cluster}. Our analysis identified four different subtypes of disease progression, corresponding to different trajectories of the diagnostic measurements' evolution over time.  We also included one additional subtype corresponding to patients who did not have PD. Appendix \ref{sec:cluster} presents a full description of these subtypes and their identification in the data.

As our black box model, we trained a random forest model to predict the progression subtype of a patient based on measurements taken during the baseline appointment and follow ups.  We considered two different prediction tasks: first, a binary prediction task to predict whether or not an individual has PD; second, a multi-class prediction task to predict one of the five identified PD progression subtypes. Further details on the construction of the black box model and its performance on these prediction tasks are given Appendix \ref{sec:exp_res}.

\begin{figure}[t]
	\centering
	\begin{subfigure}[h]{0.49\textwidth}
		\includegraphics[width=\textwidth]{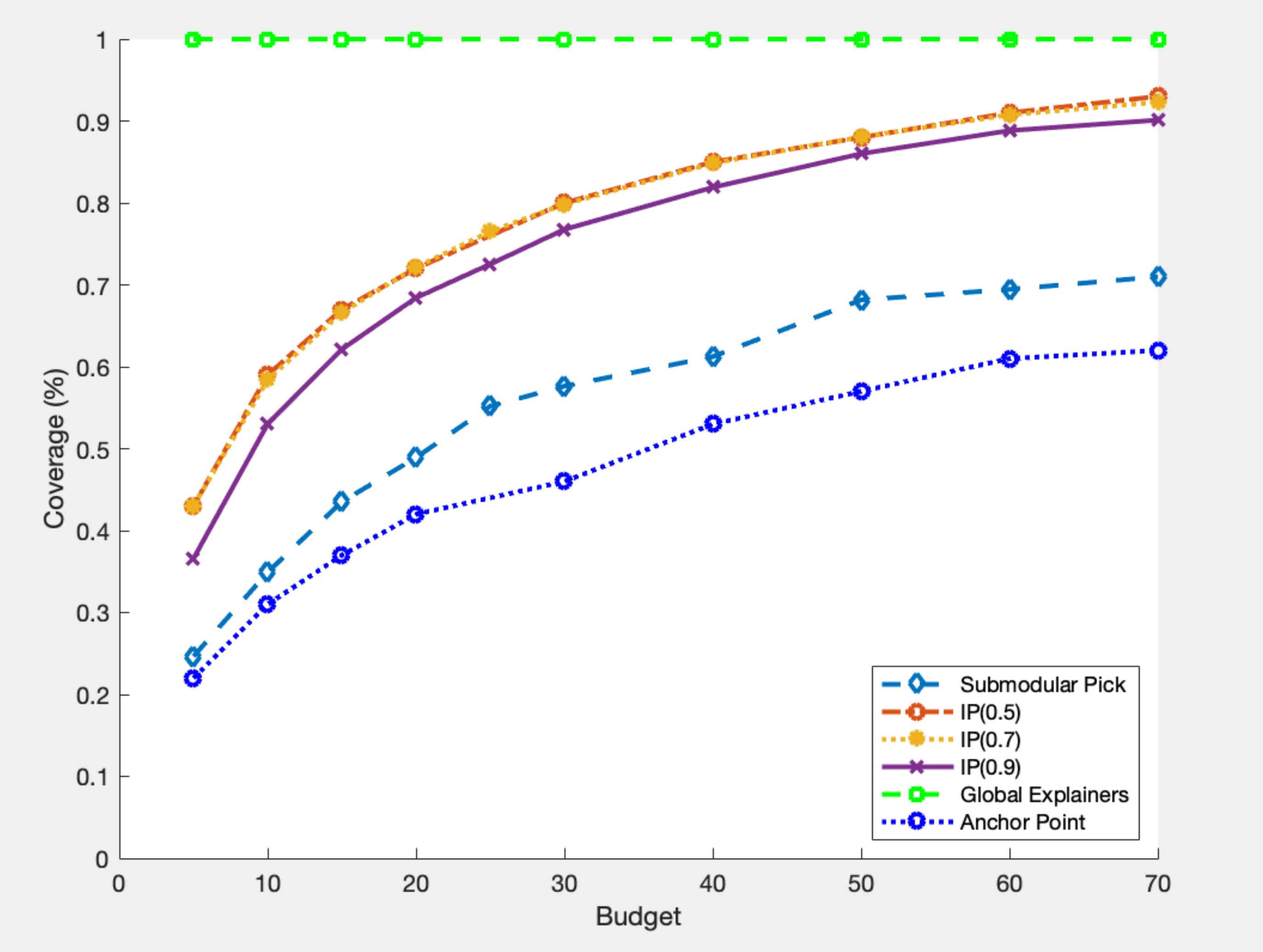}
		\caption{2 Class Coverage for PPMI Dataset}
		\label{fig:ppmi_binary_coverage}
	\end{subfigure}
	\begin{subfigure}[h]{0.49\textwidth}
		\includegraphics[width=\textwidth]{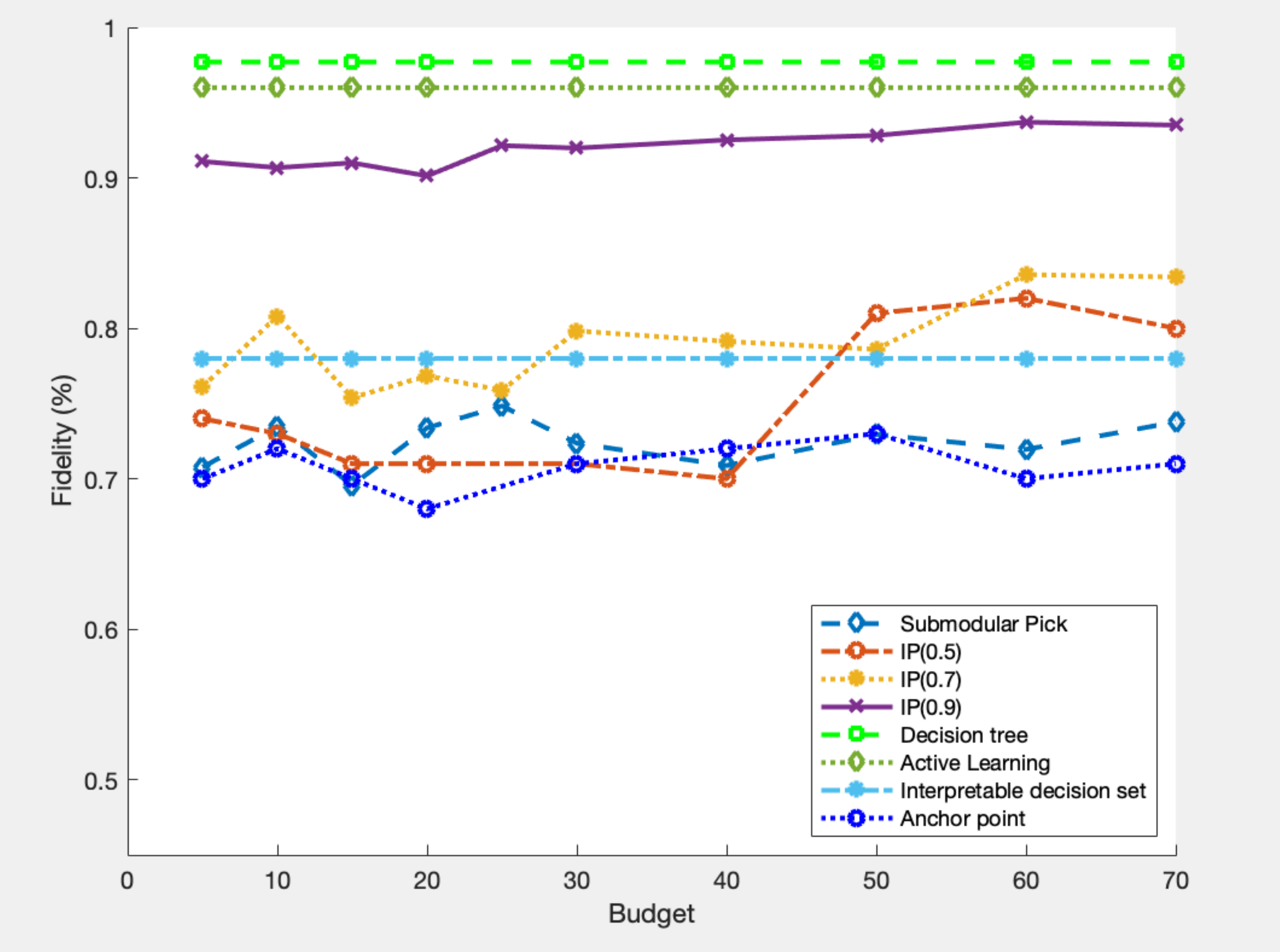}
		\caption{2 Class Fidelity for PPMI Dataset}
		\label{fig:ppmi_binary_fidelity}
	\end{subfigure}
\label{fig:ppmi_binary}
\caption{Fidelity and coverage plots for various global explainers for a random forest model trained on the PPMI data set. The x-axis corresponds to the number of constituent local explainers that are used by the aggregation methods.}
\end{figure}

\begin{figure}[h!]
	\centering
	\begin{subfigure}[h]{0.49\textwidth}
		\includegraphics[height=6cm]{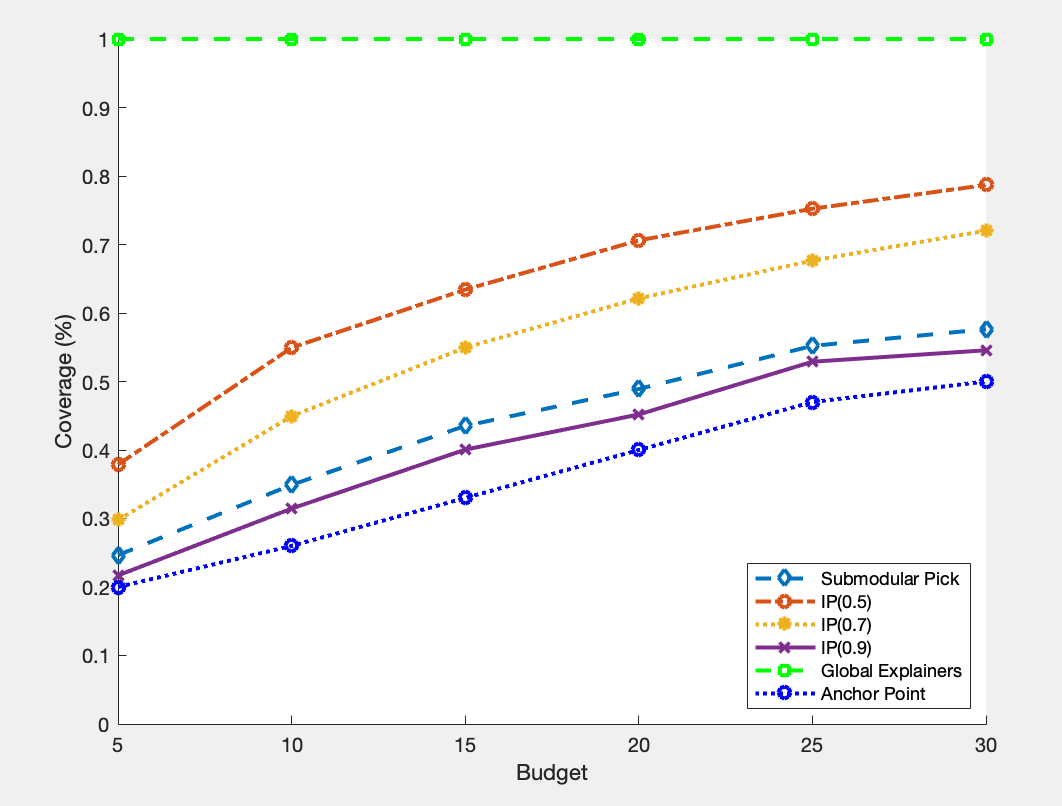}
		\caption{5 Class Coverage for PPMI Dataset}
		\label{fig:ppmi_multiclass_coverage}
	\end{subfigure}
	\begin{subfigure}[h]{0.49\textwidth}
		\includegraphics[height=6cm]{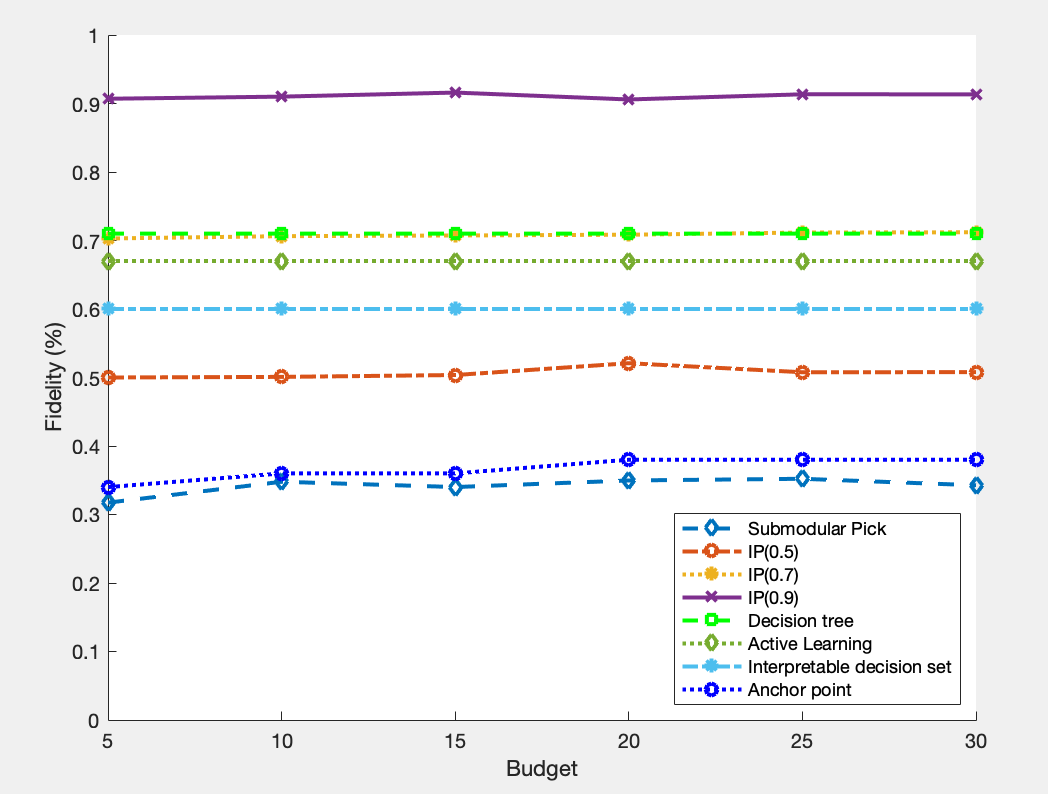}
		\caption{5 Class Fidelity for PPMI Dataset}
		\label{fig:ppmi_multiclass_fidelity}
	\end{subfigure}
	\caption{Fidelity and coverage plots for various global explainers for a random forest model trained on the PPMI data set. The x-axis corresponds to the number of constituent local explainers that are used by the aggregation methods.}
\end{figure}

We used each of the explainer methods presented above to explain the predictions made by these random forest models, and measured coverage and fidelity of these explainers.  Coverage and fidelity for the binary prediction task are shown in Figures \ref{fig:ppmi_binary_coverage} and \ref{fig:ppmi_binary_fidelity}, and similar plots for the multi-class prediction task are shown in Figures \ref{fig:ppmi_multiclass_coverage} and \ref{fig:ppmi_multiclass_fidelity}.

Figures \ref{fig:ppmi_binary_coverage} and \ref{fig:ppmi_multiclass_coverage} show that for both prediction tasks, our optimization-based aggregation algorithm obtains a higher level of coverage then both Anchor points \citep{ribeiro2018anchors} and Submodular Pick methods \citep{ribeiro2016should} across all possible local explainer budgets.   Note that when comparing coverage, global explainers are constraint to always achieve 100\% coverage.

 In terms of fidelity, Figure \ref{fig:ppmi_binary_fidelity} shows that across fidelity lower bounds of 0.7 and 0.5, our methodology performs comparably with the other aggregate explainer methods and with the explainable decision set method. When increasing our fidelity lower bound to 0.9, our method significantly outperforms these methods.  This shows that the fidelity lower bound parameter $\varphi$ in our framework allows for higher fidelity explainers given proper tuning. 
 
 In the binary case our methodology does not outperform active learning and naive decision tree in terms of fidelity or coverage; however, when considering the multi-class setting of Figure \ref{fig:ppmi_multiclass_fidelity}, we see that our framework allows for significantly higher fidelity explanations. In particular, while active learning and naive decision trees achieve a fidelity of approximately 0.7 our optimization based global classifier with $\varphi=0.9$ can achieve a fidelity of 0.9 in this case. While this is a significant increase, it does come with a cost for the coverage, as the explainer with this high fidelity only covers 40--50\% of the data, as compared to the global explainer methods of active learning and naive decision tree which cover 100\% of the data. 
 
With this in mind, our methodology allows for greater flexibility in terms of trading off explainer coverage and fidelity, especially in this multi-class case. In contrast, the pure global explainer methods do not allow for this trade-off by ensuring a hard constraint of 100\% coverage, which results in low fidelity explainers. Since our methodology out performs existing aggregation methods, this indicates that using mixed integer programming allows us to navigate the fidelity and coverage tradeoff more efficiently.

\subsubsection{Comparison of Local Explainer Performance in Aggregate}

\begin{figure}[tbh]
	\centering
	\begin{subfigure}[h]{0.49\textwidth}
		\includegraphics[height=6cm]{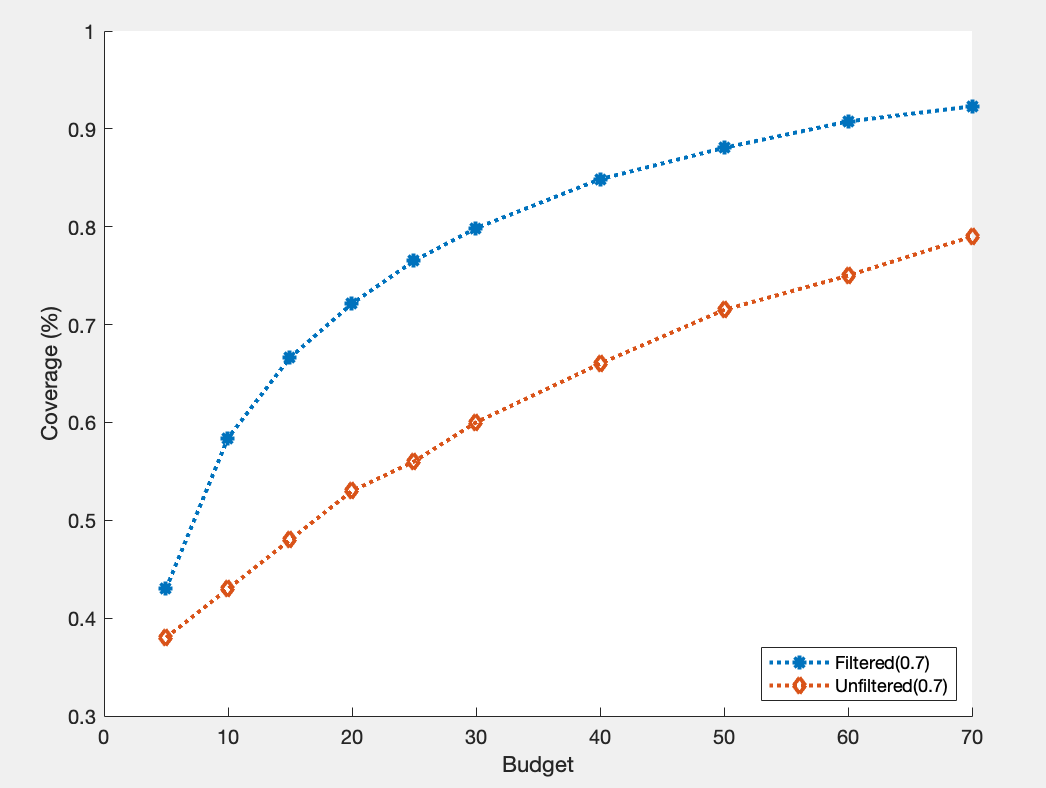}
		\caption{comparison of the coverage local methods for two classes}
		\label{fig:ppmi_binary_fixglobalcoverage}
	\end{subfigure}
	\begin{subfigure}[h]{0.49\textwidth}
		\includegraphics[height=6cm]{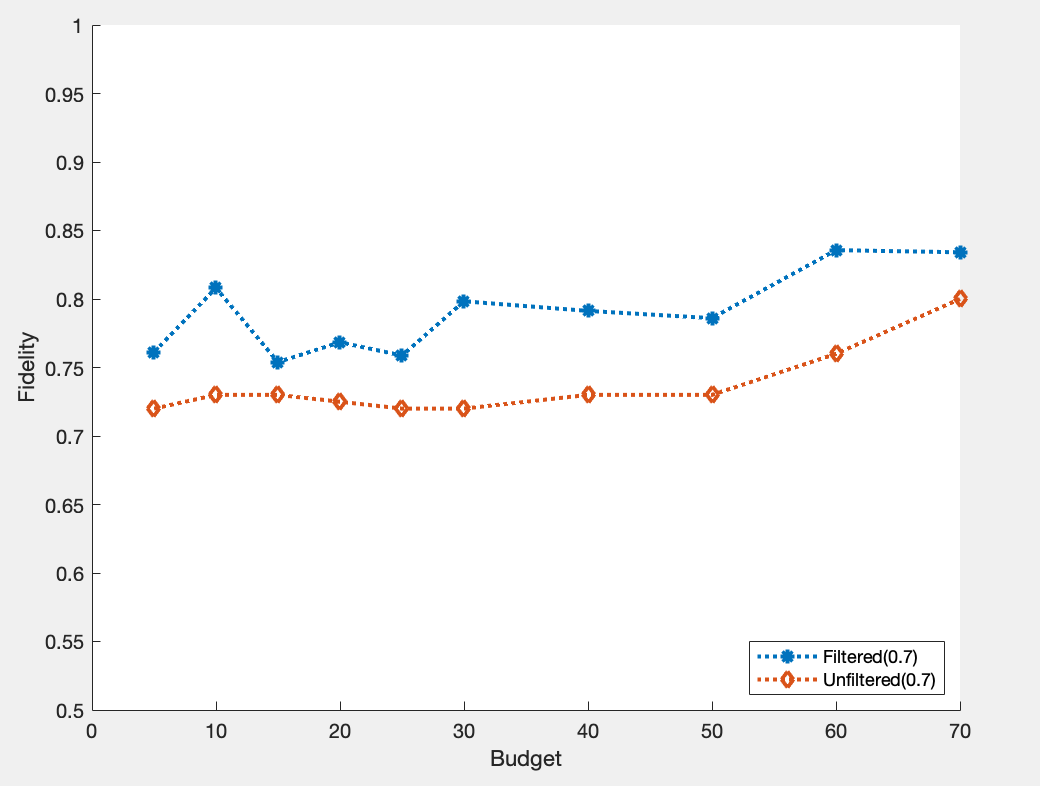}
		\caption{comparison of the fidelity of local methods for two classes}
		\label{fig:ppmi_binary_fixglobalfidelity}
	\end{subfigure}
\label{fig:ppmi_binaryfixglobal}
\caption{Fidelity and coverage plots for an IP based explainer aggregate using both an information filter based local explainer (labeled filtered) and LIME type local explainer (labeled unfiltered). These plots are for a binary classification task. The x-axis corresponds to the number of constituent local explainers that are used by the aggregation methods.}
\end{figure}

\begin{figure}[tbh]
	\centering
	\begin{subfigure}[h]{0.49\textwidth}
		\includegraphics[height=6cm]{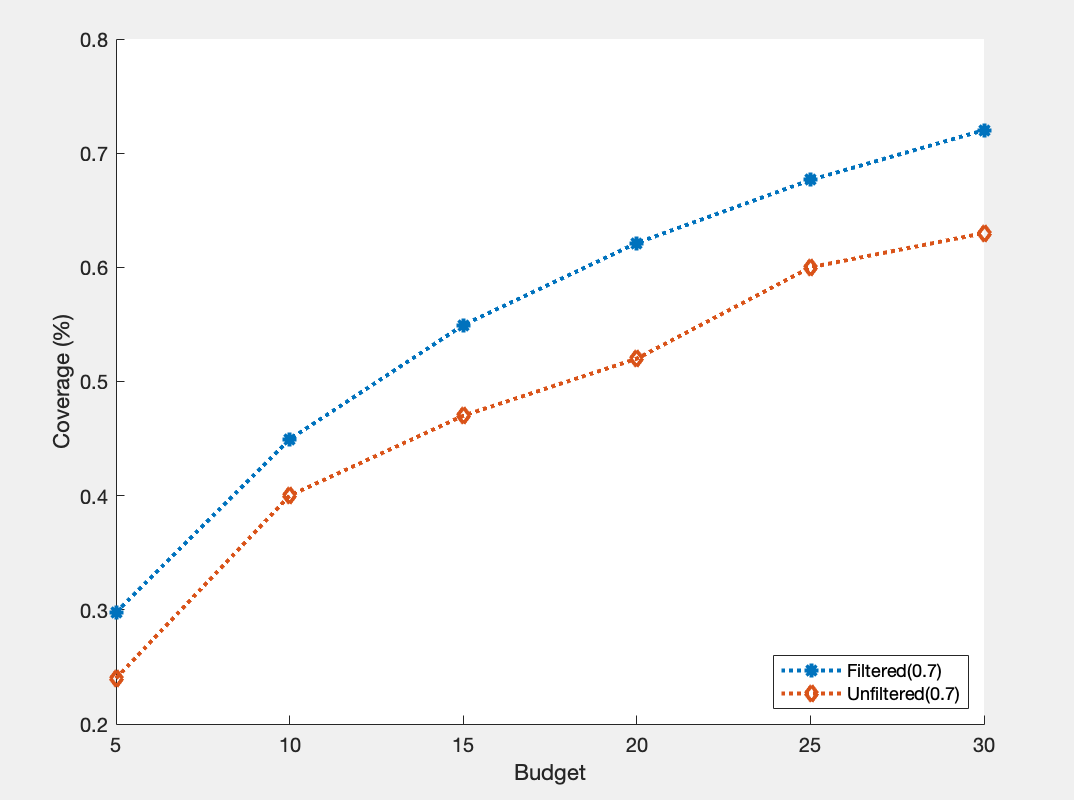}
		\caption{comparison of the coverage local methods for five classes}
		\label{fig:ppmi_multi_fixglobalcoverage}
	\end{subfigure}
	\begin{subfigure}[h]{0.49\textwidth}
		\includegraphics[height=6cm]{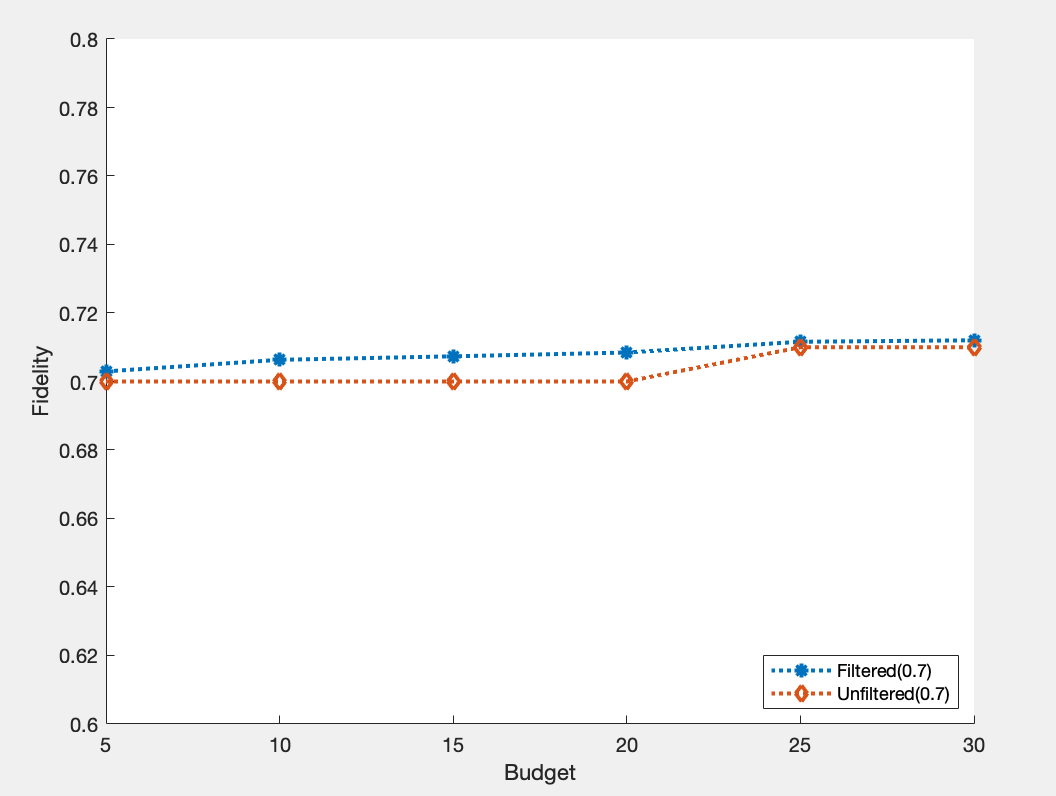}
		\caption{comparison of the fidelity of local methods for five classes}
		\label{fig:ppmi_multi_fixglobalfidelity}
	\end{subfigure}
\label{fig:ppmi_multifixglobal}
\caption{Fidelity and coverage plots for an IP based explainer aggregate using both an information filter based local explainer (labeled filtered) and LIME type local explainer (labeled unfiltered). These plots are for a multiclass classification task. The x-axis corresponds to the number of constituent local explainers that are used by the aggregation methods.}
\end{figure}

To evaluate the performance of our proposed local explainer methodology in the context of explainer aggregation, we considered the impact on aggregate fidelity and coverage of our aggregate explainer using different base local explainers. For this experiment we used our IP methodology as the mode of aggregation and evaluated the difference between using our proposed information filter based local explainer (labeled in the plots as ``filtered'') and LIME (labeled in the plots as ``unfiltered'') as the base local explainers to be aggregated. For these experiments, we fixed the lower bound on fidelity of the IP at 70\% and plotted both the coverage and fidelity of the aggregate with different explainer budgets for both binary prediction and multi-class prediction.

Figures \ref{fig:ppmi_binary_fixglobalcoverage} and \ref{fig:ppmi_binary_fixglobalfidelity} show the coverage and fidelity comparisons for the binary prediction class. We see that the use of our information-filter-based local explainer provides a better coverage and roughly 4\% higher fidelity score then those obtained by our aggregation method in conjunction with LIME across all budget levels. These results indicate that our prosed local explainer methodology leads to aggregate explainers that include both simpler component explainers, and can achieve improved coverage and fidelity in the binary classification case.  

Figures \ref{fig:ppmi_multi_fixglobalcoverage} and \ref{fig:ppmi_multi_fixglobalfidelity} show the coverage and fidelity  comparisons for the multi class prediction task. In this setting we again see that our proposed local explainer provides improved coverage and fidelity across all potential aggregate budgets. The advantage in the multi-class setting is less pronounced than in the binary prediction case, but our method still provides on average $5\%$ improvement in coverage over LIME for the resulting aggregate explainer.


\subsection{Geriatric Activity Classification}
For the second set of experiments we used a data set of Geriatric Activity based on the study conducted by \citep{torres2013sensor}. The main goal of this study was to provide ways of potentially reducing the likelihood of falls for geriatric individuals by classifying their activities when transferring beds. Generally, the highest risk for geriatric patients to fall is when getting out of bed so various sensors were deployed to detect whether an individual was attempting to leave their bed and detect other potentially risky activity. For this particular study, the authors used a novel wearable and environmental sensor which they validated with 14 individuals aged 66--86. The goal was to use this sensor data to classify between three different activities, namely laying in bed, sitting in the bed, and getting out of the bed. To generate the data set, each of the participants was asked to perform a random set of five activities which ranged between the three potential activity classes.

Much like in the case of the PPMI data set, we trained a random forest model to classify between the various activity classes that we used to extract global explainers. However, unlike the PPMI experiments, since there was no straight forward way to convert the multiclass classification task of detecting the different activities into a binary classification task we only performed the experiments for the multiclass case. The results for all explainer methods can be seen in Figures \ref{fig:geriatric_coverage} and \ref{fig:geriatric_fidelity}. Much like in the case for the PPMI data set, we note our methodology out performs other aggregation based global explainers with respect to coverage across all budgets and fidelity lower bounds; however, it is still not obtaining 100\% coverage like the pure global explainer methodologies. In terms of fidelity, much like in the multiclass case of the PPMI data, our methodology out performs all other global explainers, with active learning being close to on par with our performance. This further suggests that using this form of optimization based local explainer aggregation is well suited to explaining multiclass predictions regardless of the underlying data set.

\begin{figure}[h]
	\centering
	\begin{subfigure}[h]{0.49\textwidth}
		\includegraphics[height=6cm]{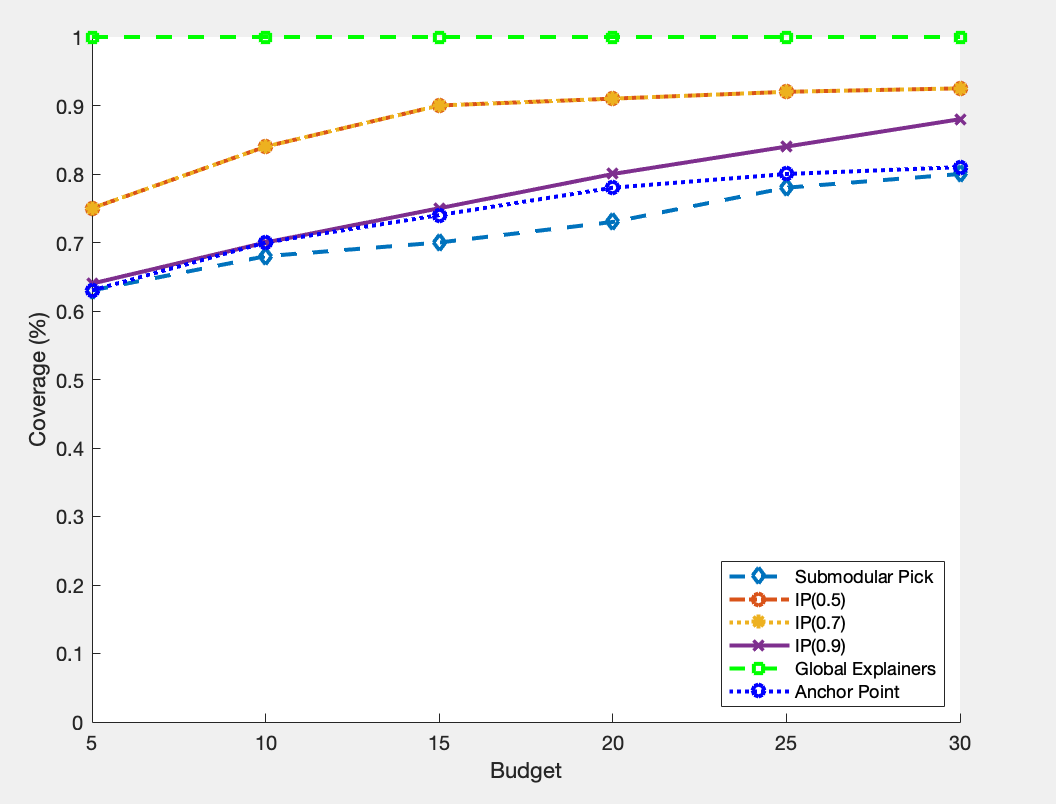}
		\caption{5 Class Coverage for Geriatric movement Dataset}
		\label{fig:geriatric_coverage}
		
	\end{subfigure}
	\begin{subfigure}[h]{0.49\textwidth}
		\includegraphics[height=6cm]{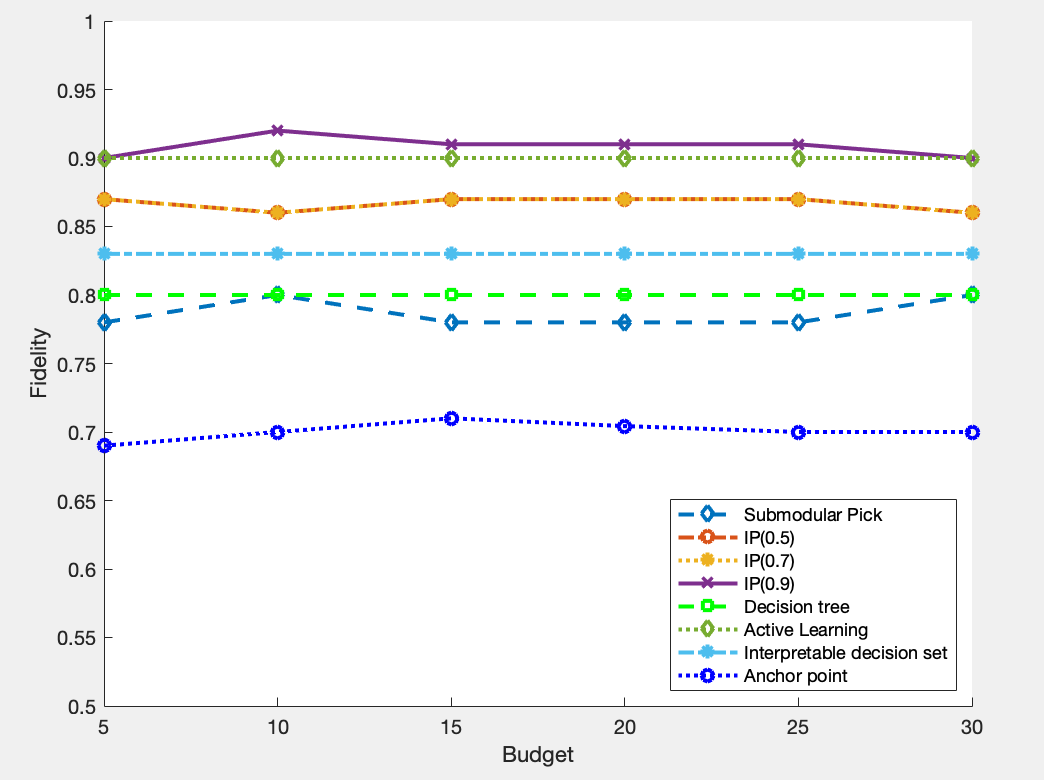}
		\caption{5 Class Fidelity for Geriatric movement Dataset}
		\label{fig:geriatric_fidelity}
	\end{subfigure}
	\caption{Fidelity and coverage plots for various global explainers for a random forest model trained on the Geriatric Movement data set. The x-axis corresponds to the number of constituent local explainers that are used by the aggregation methods.}
        
\end{figure}

%% file: conc.tex
\section{Discussion on Societal Implications}

Our aggregate explainer methodology provides explicit parameters that allow practitioners to clearly trade off among explainer coverage, fidelity, and interpretability.  We note that in this trade off, low fidelity also results in low transparency, because the explanations offered by the explainer diverge significantly from the black box predictions that are meant to be explained. For example, explainers used for diagnostics might want to weigh more towards coverage, while explainers used for prediction transparency might want to weigh more towards fidelity. Navigating this tradeoff efficiently is critical to ensure that practitioners can correctly inform users or patients of the ML predictions. These contributions are particularly valuable in medical applications or other settings where informed consent is required.


%% file: appcluster.tex
\section{Clustering Methodology and PPMI Dataset}
\label{sec:cluster}
PD is a complex disorder, and is often expressed differently by different patients, which has motivated the need to create PD sub-types to better direct treatment. While many existing data-driven methods focus on clustering patients based on their baseline measurements \citep{fereshtehnejad2017subtypes}, we propose clustering patients using the trajectory of how their symptoms progress. 

We will use data collected in the PPMI study \citep{PPMI}, which is a long run observational clinical study designed to verify progression markers for PD. To achieve this aim, the study collected data from multiple sites and includes lab test data, imaging data, genetic data, among other potentially relevant features for tracking PD progression. The study includes measurements of all these various values for the participants across 8 years at regularly scheduled follow up appointments. The complete data set contains information on 779 patients, and included 548 patients diagnosed with PD or some other kind of Parkinsonism and 231 healthy individuals as a control group.

\subsection{Determination of Criterion and Cluster Analysis}\label{s.crit}
Since there is significant heterogeneity in how PD symptoms are expressed, there also is no agreement on a single severity score or measurement that can be used as a surrogate for PD progression. Thus instead of considering a single score, we will model the severity of the disease as a multivariate vector, and the disease progression as the trajectory of this vector through a multidimensional space. Using the PPMI data \citep{PPMI} and other previous literature on PD progression \citep{rao2006parkinson,martinez2017rating,bhat2018parkinson}, we considered the following measures of severity to model disease progression:
\begin{itemize}
	\item Unified Parkinson's Disease Rating Scale (UPDRS) II \& III \citep{martinez1994unified}: The UPDRS is a questionnaire assessment that is commonly used to track symptoms of PD by an observer.  It consists of four major sections, each meant to measure a different aspect of the disease.  These sections are: (I) Mentation Behavior and Mood, which includes questions related to depression and cognitive impairment; (II) Activities for Daily Living, which includes questions related to simple daily actions such as hygiene and using tools; (III) Motor Examination, which includes questions related to tremors and other physical ticks; and (IV) Complications of Therapy, which attempts to assess any adverse affects of receiving treatment.  For our analysis we focused on the aggregate scores of sections II and III of the UPDRS to track physical symptoms of the disease.
	\item Montreal Cognitive Assessment (MoCA) \citep{nasreddine2005montreal}: Although not exclusively used for PD, the MoCA is a commonly used assessment for determining cognitive impairment and includes sections related to attention, executive functions, visual reasoning, and language. For our analysis, we used the MoCA scores of the individual patients as surrogates for their cognitive symptoms.
	\item Modified Schwab and England Activities of Daily Living Scale (MSES) \citep{siderowf2010schwab}: The MSES is a metric used to measure the difficulties that individuals face when trying to complete daily chores due to motor deficiencies. This assessment is generally administered at the same time as the UPDRS and is often appended as a section V or VI. We used this score as a measure of how much autonomy the patients experience based on their symptoms. 
\end{itemize}

We formed the empirical trajectory of these scores for each patient using the values measured during the patients' participation in the PPMI study \citep{PPMI}. For our cluster analysis we used longitudinal measurements that were taken across the first seven visits of the study corresponding to  a period of 21 months, where the first measurement formed the patient's baseline, and the next five measurements were taken at follow up visits at regular three month intervals; the final measurements were taken after six months.  We chose this timeline for our analysis because participation was high among all participants in the study during this period, so we did not have to exclude any patients, and visits were more frequent to better capture disease progression over time. After these seven measurements, follow-up visits were scheduled too infrequently to provide useful trajectory modeling information. 

We used these trajectories to cluster the patients together into progression sub-types. The main motivation for this approach is that if patients' severity scores progress in a similar way, then it may identify a useful sub-type for treatment design. Only patients diagnosed with PD were included in the cluster analysis, since we are interested in finding useful sub-types of disease progression. Each trajectory was then flattened out as a 28 dimensional vector, with the first four entries corresponding the measurements at baseline, the next four for the 3 month follow up, and so on. Using scikit-learn and Python 3.7, we performed $k$-means clustering on these trajectories to define our sub-types \citep{pedregosa2011scikit,friedman2001elements}. Using cross validation and the elbow method (as seen in Figure \ref{fig:elbow_plotl} in the appendix), we determined that there are four potential sub-types of disease progressions for the PPMI participants. We label these as: moderate physical symptoms cognitive decline cluster (Group 0), stagnant motor symptoms autonomy decline cluster (Group 1), motor symptom dominant cluster (Group 2), and moderate symptoms cluster (Group 3). The names we assigned to each individual cluster were given by the observed mean trajectories of the relevant scores for individuals that were classified into a particular cluster as shown in Figure \ref{fig:mean_traj}.

\begin{figure}[h]
	\centering
	\includegraphics[scale=0.65]{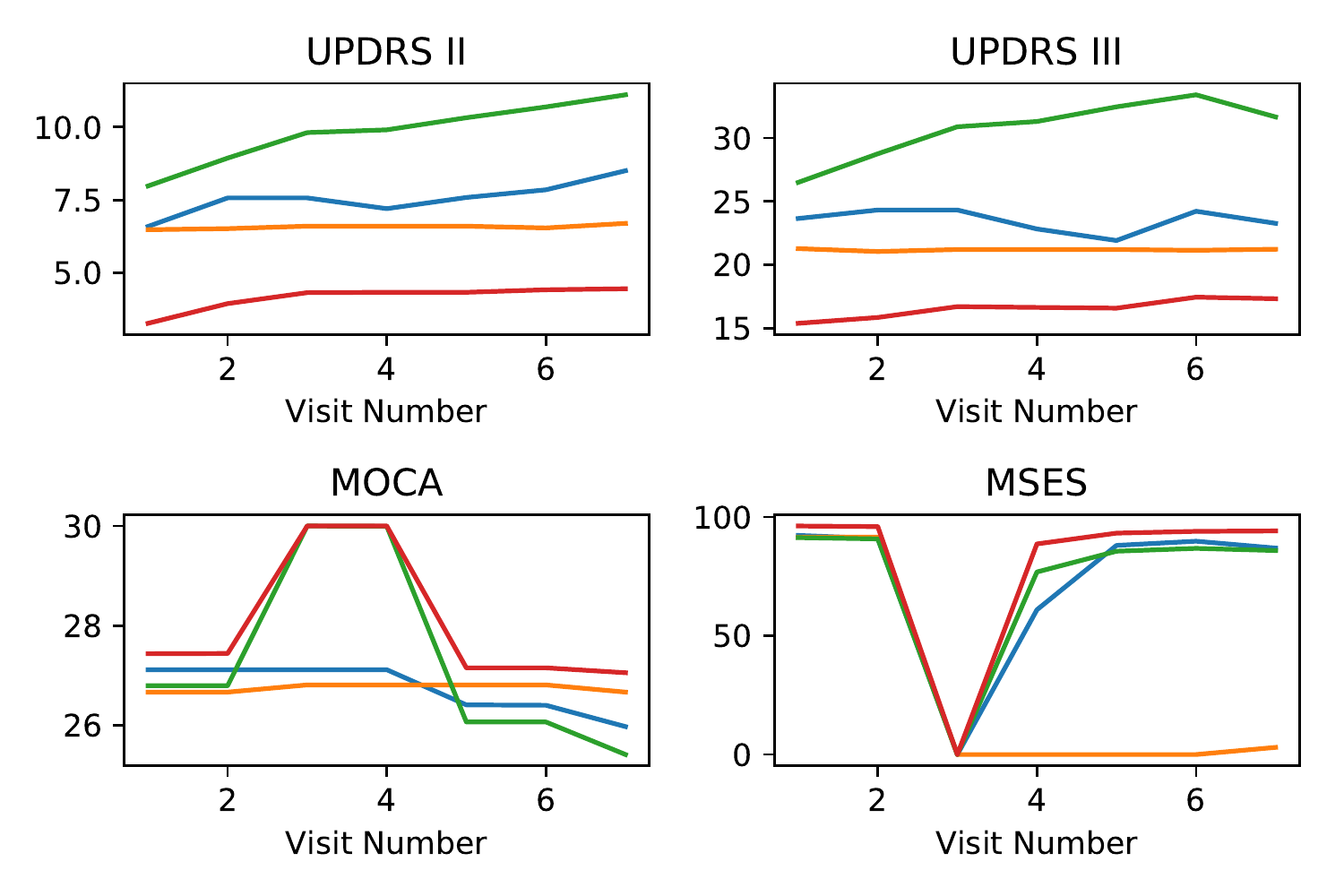}
	\caption{Mean trajectory progression for given score by cluster. Blue corresponds to Group 0, orange corresponds to Group 1, green corresponds to Group 2, and red corresponds to Group 3. The y-axis of each plot the is numerical value of the corresponding disease severity measure.}
	\label{fig:mean_traj}
\end{figure}

In Figure \ref{fig:projections} we show two 2-dimensional projections of the different cluster groups. Figure \ref{fig:pca_cluster} shows the projection onto the first two principal components of the data using PCA \citep{friedman2001elements}; this projection method is meant to preserve linear relationships among data points as well as distances between data points that are far apart.  The projection shown in Figure \ref{fig:tsne_cluster} corresponds to the tSNE projection of the data onto a two-dimensional space \citep{maaten2008visualizing}, this projection method was designed with manifolds in mind and is meant to preserve close distances (i.e., data points close in the tSNE projection should be also close in the higher dimensional space). Note that in both projections our resulting clusters are distinct and do not significantly overlap. 

\begin{figure}[h]
	\centering
	\begin{subfigure}{0.49\textwidth}
		\includegraphics[width=\textwidth]{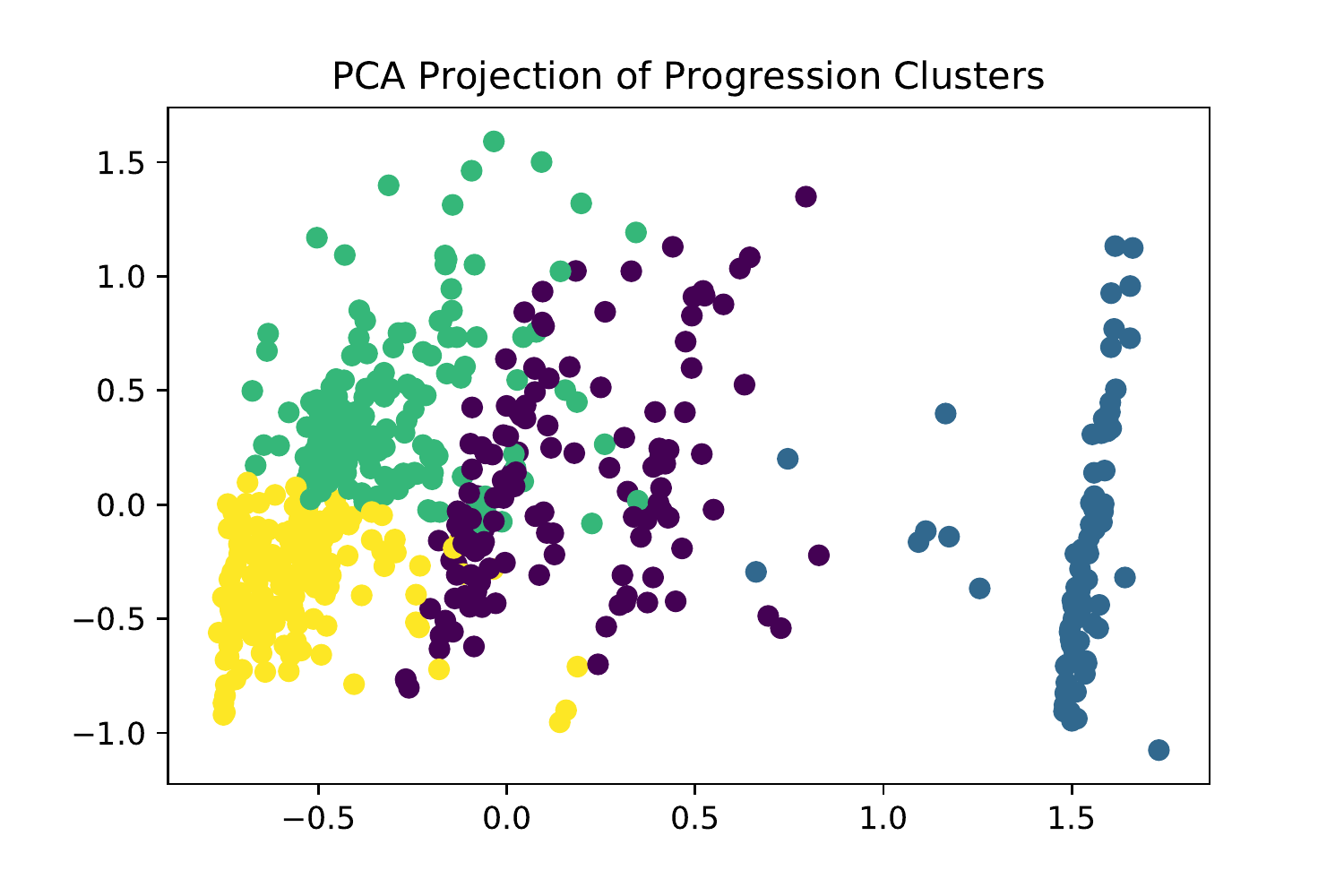}
		\caption{PCA Projection}
		\label{fig:pca_cluster}
	\end{subfigure}
	\begin{subfigure}{0.49\textwidth}
		\includegraphics[width=\textwidth]{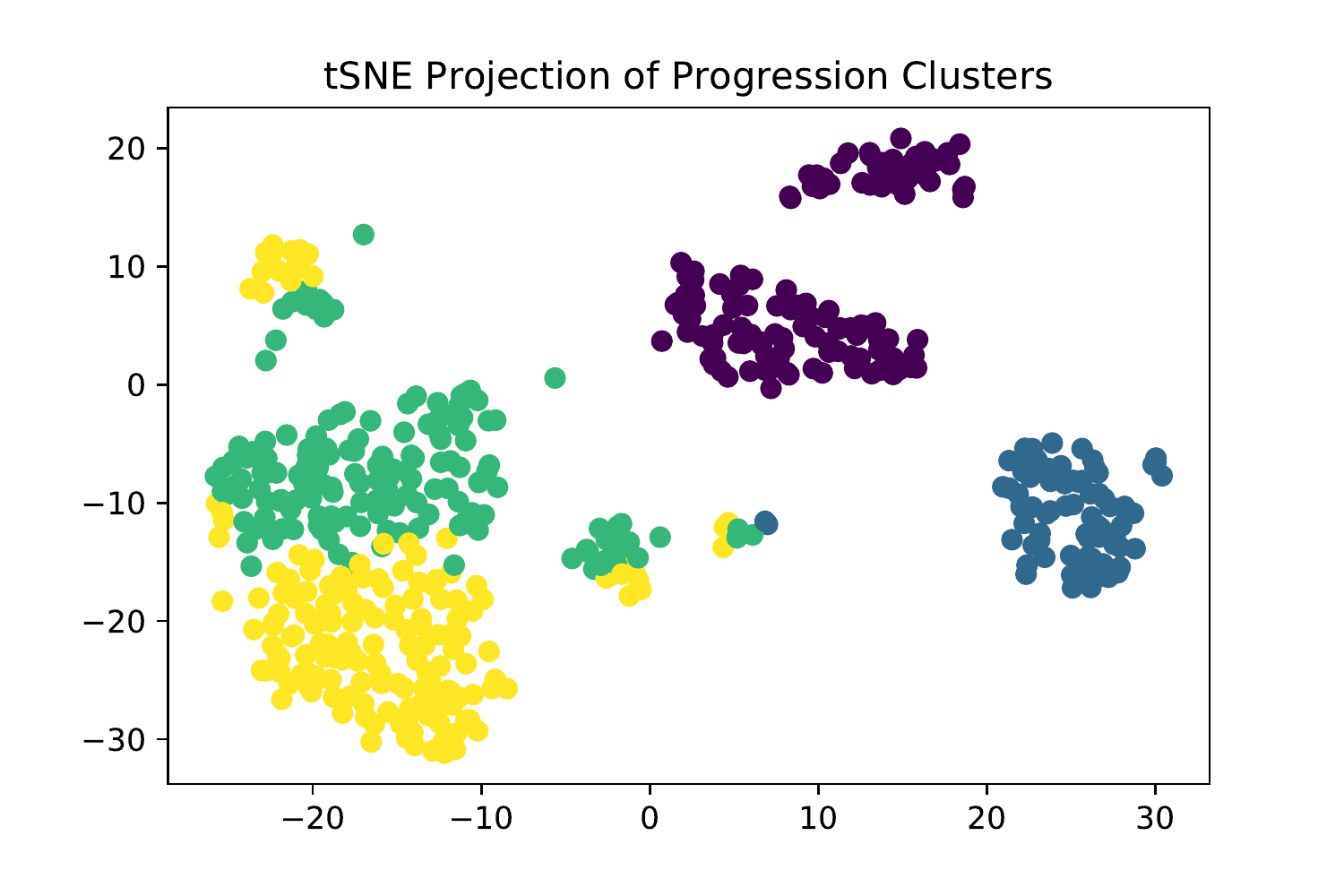}
		\caption{tSNE Projection}
		\label{fig:tsne_cluster}
	\end{subfigure}
	\caption{Two different 2-dimensional projections for visualizing trajectory clusters. Purple corresponds to Group 0, blue corresponds to Group 1, green corresponds to Group 2, and yellow corresponds to Group 3. }
	\label{fig:projections}
\end{figure}

\subsection{Validation of Clusters}
To test whether these clustered sub-types provide additional insight into the health of the patients, we performed several statistical comparisons of each patients' characteristics at baseline across all four sub-types plus healthy patients, to determine if there were any statistically significant differences. The results and values of these comparisons are presented in Table \ref{table:stat_comp} below.

\begin{table*}[tbh]
	\begin{tabular}{ll|llllll}
		& & Group 0   & Group 1   & Group 2   & Group 3  & Healthy   & p value                          \\
		\hline 
		\multicolumn{2}{l|}{Lymphocytes}                & 1.643$^m$    & 1.749     & 1.642$^n$    & 1.704$^p$   & 1.850$^{mnp}$  & 0.01                              \\
		\multicolumn{2}{l|}{REM Sleep Score}            & 5.549$^{de}$   & 1.892$^{dfgh}$ & 5.969$^{fij}$  & 5.087$^{gik}$ & 3.247$^{ehjk}$ & \textless{}0.001                   \\
		\multicolumn{2}{l|}{UPDRS part II}              & 6.594     & 6.482     & 7.981     & 3.272    & N/A       & \textless{}0.001                   \\
		\multicolumn{2}{l|}{UPDRS part III}             & 23.654    & 21.277    & 26.503    & 15.382   & N/A       & \textless{}0.001                   \\
		\multicolumn{2}{l|}{Schwab \& England Score}   & 92.256    & 91.506    & 91.321    & 96.214   & N/A       & \textless{}0.001                   \\
		\multicolumn{2}{l|}{Age}                        & 58.925$^a$   & 60.446    & 62.912$^{abc}$ & 58.387$^b$  & 59.571$^c$   & 0.02                               \\
		\multirow{3}{*}{Olfactory} & Anosmia   & 46        & 10        & 57       & 41        & 6                & \textless{}0.001 \\
		& Hyposmia  & 68        & 11        & 91       & 98        & 68               &                  \\
		& Normosmia & 19        & 5         & 11       & 34        & 122              &                  \\
		\multicolumn{2}{l|}{Race White}                 & 95.49\%   & 93.98\%   & 94.34\%   & 94.22\%  & 94.37\%   & 0.99                               \\
		\multicolumn{2}{l|}{Gender Male}                & 67.67\%   & 57.83\%   & 65.41\%   & 63.01\%  & 65.80\%   & 0.63                               \\
		\multicolumn{2}{l|}{Geriatric Depression Score} & 5.391     & 5.069     & 5.270     & 5.231    & 5.168     & 0.68                              
	\end{tabular}
	\caption{Comparison of baseline and screening measurements between clusters. p-values labeled in the table represent difference between all groups, and significant pairwise comparisons using a two sample T-test are marked by superscripts with p-values a-0.008; b-0.001; c-0.02;d,e,f,g,h,i,j,k-$<$0.001, m-0.003;n-0.004;p-0.04
	}
	\label{table:stat_comp}
\end{table*}

As seen in Table \ref{table:stat_comp}, many of the key screening measurements of the populations from the different clusters are significantly different, implying our clusters are informative about the health of individuals. In particular, we note that Group 0---which corresponds to moderate physical symptoms with cognitive decline---tends to be younger on average then the other groups, indicating this group may contain many more individuals with early onset PD. Moreover, the sub-types vary substantially in their sleep score and olfactory evaluation, which are both measures that have previously been shown to be strong indicators of PD \citep{rao2006parkinson} indicating that these progression sub-types are sensitive to these important predictors.

Overall, the comparisons shown in Table \ref{table:stat_comp} show that our data driven clusters are not only informative when comparing different forms of disease progression, but also correspond to variations in screening measurements. Based on this analysis, we believe that using screening data to predict these clusters could lead to clinically significant insights that can help with treatment.

%% file: applocal.tex
\section{Local Explainer Algorithm}
\label{sec:loc_exp}
After identifying the four disease progression sub-types, we would like to predict which kind of disease progression an individual might experience, given measurements collected during a screening visit. As we will show in our experiments in Section \ref{sec:exp_res}, this task is best performed by complex black box models such as artificial neural networks (ANN) and bagged forests. This means that while the prediction may be accurate, it will not be easily explained, which make such models difficult to use for diagnosis recommendations. Our goal is to instead develop a method that trains simple auxiliary explainer models, and can still accurately describe the relationship between the data and the model output within a small region of a given prediction.

This methodology is known as \emph{training local explainer models} and has been shown to be useful in understanding black box predictions \citep{ribeiro2016should,ribeiro2018anchors}.  One of the key tradeoffs in generating model explanations is that of \emph{fidelity}---how well the explainer approximates the black box model---and \emph{interpretability}---how easy it is for a practitioner to trace the predictions of the model. In contrast to previous literature which has proposed the use of regularization to achieve this goal, we propose directly computing locally significant features using an information filter. Generally, computing such filters can be computationally expensive and requires the use of  numerical integration; however, one of  our main contributions in this paper is to introduce an efficient algorithm for filtering out less significant features. This methodology will allow us to train local explainers that are significantly less complex then those that use regularization, with better fidelity. 


\subsection{Local Explainer Notation}
\label{sec:notation}
Before proceeding to our discussion on the local explainer method, we will first establish some technical notation. We assume that for each patient $i = 1,...,n$ we have an ordered pair $(x_i,y_i)$, where $x_i \in \mathcal{X} \subseteq \mathbb{R}^m$ are the features values of the patient and $y_i \in \mathcal{L} \subseteq \mathbb{Z}$ is the corresponding class label generated by a black box model $f$. Through our analysis we will also refer to this set of points through matrix notation where $X \in \mathcal{X}^n \subseteq  \mathbb{R}^{m\times n}$ is the feature value matrix and $y \in \mathcal{L}^n \subseteq \mathbb{Z}^n$ is the vector of class labels, where each row in these matrices corresponds to a single patient's data. For our analysis we assume that $\mathcal{X}$ is a compact set. Let $\Phi= \{1,...,m\}$ be the set of features, and it may also be used to denote the index set of the features. This set can be partitioned into two sets $\Phi_c,\Phi_b \subseteq \Phi$ that represent the set of continuous and binary features respectively.

Furthermore we define the set-valued function $\Phi^*: \mathcal{X} \rightarrow \Phi$ as the function which extracts the minimum set of necessary features to accurately predict the class of a point $x$. Namely,
\begin{equation}
\Phi^*(x) = \argmin_{\varphi \subseteq \Phi}\{ |\varphi| : p(y|x) = p(y|x[\varphi])\},
\end{equation}
where $x[\varphi]$ is an indexing operation that maintains the values of $x$ but only for the features in $\varphi$, and $p$ is the conditional probability mass function of the labels $y$ given the observation of some features. Specifically, if a feature index is not included $\Phi^*(x)$, then it is not required to understand the particular 
label of $x$. In addition, we will denote the ball around a point $x$ of radius $r$ with respect to a metric $d$ as $\mathcal{B}(x,r,d)$.

Finally, a key feature of the explainer training method we propose includes the use of \emph{mutual information}. In information theory, mutual information is a quantity that measures how correlated two random variables are with one another. If $X,Y$ are two random variables with joint density $p$ and marginal densities $p_x,p_y$, then the mutual information between $X$ and $Y$ is denoted $I(X;Y)$ and calculated as:
\begin{equation}
I(x;y) = \mathbb{E}\log\frac{p(X,Y)}{p_x(X),p_y(Y)} = \int_{x}\int_{y} p(x,y) \log\frac{p(x,y)}{p_x(x),p_y(y)}dx dy. 
\end{equation}
If $X$ and $Y$ are independent then $I(X;Y) = 0$; otherwise $I(X;Y) > 0$, meaning that $X$ contains some information about $Y$. A similar quantity can be computed using a conditional distribution on another random variable $Z$, known as the \emph{conditional mutual information} and denoted $I(X;Y|Z)$.

\subsection{Local Explainer Algorithm Description}
Our main local explainer algorithm extends previous local explainer methods such as LIME \citep{ribeiro2016should} by restricting the sampling region around the prediction, and including an information filter to ensure that fewer features are included in the final explainer mode.

Our general local explainer is formally presented in Algorithm \ref{alg:loc_exp}, but we will give a brief overview of its operations here. The algorithm takes in hyper-parameters including number of points $N$ to be sampled for training the explainer, a distance metric $d$, and a radius $r$ around the point $\bar{x}$ being explained. First the algorithm samples $N$ points uniformly from within a $r$ radius of $\bar{x}$; we call this set of points $T(\bar{x})$. Depending on the distance metric being used this can often be done quite efficiently, especially if the features are binary valued or an $\ell^p$ metric is used \citep{barthe2005probabilistic}.  Then using the sampled points, the algorithm uses the Fast Forward Feature Selection (FFFS) algorithm as a subroutine (formally presented in Section \ref{sec.fffs} and Appendix \ref{app.fffs}), which uses an information filter to remove unnecessary features and reduce the complexity of the explainer model. The FFFS algorithm uses an estimate of the joint empirical distribution of $(T(\bar{x}),f(T(\bar{x}))$ to select the most important features for explaining the model's predictions in the given neighborhood. We denote this set of features $\hat{\Phi}$. Then, using these features and the selected points, the explainer model $g$ is trained by minimizing an appropriate loss function that attempts to match its predictions to those of the black box model.  In principle a regularization term can be added to the training loss of explainer $g$. However, through our empirical experiments in Section \ref{sec:exp_res} we found that FFFS typically selected at most five features, so even the unregularized models where not overly complex.

\begin{algorithm}
	\begin{algorithmic}[1]
		\caption{Local Explainer Training Algorithm}
		\Require sampling radius $r$, number of sample points $N$, black box model $f$, data point to be explained $\bar{x}$, and loss function $L$ for the explainer model $(\bar{x},\bar{y})$
		\State Initialize $T(\bar{x}) = \emptyset$
		\For {$j = \{1,...,N\}$}
		\State Sample $x \sim U(\mathcal{B}(\bar{x},r,d))$
		\State  $T(\bar{x}) \leftarrow T(\bar{x}) \cup x$
		\EndFor
		\State Obtain $\hat{\Phi}(\bar{x}) = \text{FFFS}(T(\bar{x}), \Phi, f)$
		\State Train $g = \argmin_{\hat{g} \in \mathcal{G}}\{\sum_{x\in T(\bar{x})}L(f(x) - \hat{g}(x[\hat{\Phi}])) \}$
		\State \Return g
	\end{algorithmic}
\end{algorithm}

\subsection{Fast Forward Selection Information Filter}\label{sec.fffs}
A key step in our algorithm is the use of a mutual information filter to reduce the number of features that will be included in the training of the local explainer. Mutual information filters are commonly used in various signal processing and machine learning applications to assist in feature selection \citep{brown2012conditional}. However, these filters can be quite challenging to compute depending on the structure of the joint density function of the features and labels, and can require the use of (computationally expensive) numerical integration. We counteract this by considering an approximation of the density function, using histograms to calculate continuous features. When multiple combinations of features need to be considered as in our setting, the problem of finding the maximum-information minimum-sized feature set is known to be computationally infeasible \citep{brown2012conditional}. As such, our proposed method for computing the filter includes a common heuristic known as \emph{forward selection}, which essentially chooses the next best feature to be included in the selected feature pool in a greedy manner.  Using this method alone would still require recomputing the conditional distribution of the data based on previously selected features, which can result in long run times for large $N$. However, using some prepossessing techniques, we can show that these quantities can be stored efficiently using a tree structure, which allows quick computation of the filter.

The general idea of the FFFS algorithm is to consider the feature selection process as a tree construction. Part of this construction relies on an estimate of the empirical density of the features as a histogram with at most $B$ bins and preprocessed summary tensor $M \in \{0,1\}^{B\times |\Phi| \times N}$ which indicates which bin of the histogram a feature value for a particular data point lays in. For each entry, $M[b,\varphi,x] = 1$ if the value of feature $\varphi$ at point $x$ falls in the bin $b$. Otherwise, $M[b,\varphi,x] = 0$.  The depth of the tree represents the number of selected features and each node of the tree is a subset of $T(\bar{x})$. For instance, at the beginning of the selection process, we have a tree with exactly one node $R$ where $R=T(\bar{x})$.  Assume binary feature $\varphi_1$ is selected in the first round. Then two nodes $a,b$ are added under $R$, where $a = \{x_j: M[1,\varphi_1,j]=1\}$ and $b = \{x_j: M(2,\varphi_1,j)=1\}$. In the second round, we use the partition sets $a,b$ to compute the mutual information instead of the complete set $R$. The set $a$ is used for computing $\hat{p}(\varphi|\varphi_1=1),\; \hat{p}(y|\varphi_1=1) \text{, and }\hat{p}(\varphi;y|\varphi_1=1)$, while $b$ is used when the condition is $\varphi_1=2$. In each round the leaves $\mathcal{L}$ of the current tree represent the set of partition sets corresponding to all random permutation of selected features information. Therefore, $\mathcal{L}$ provides us sufficient information for calculating the desired mutual information. As shown in Algorithm \ref{alg:sf}, the algorithm only outputs the leaves $\mathcal{L}$, not the entire tree. The main algorithmic challenge is to efficiently calculate the marginal distributions $(\hat{p}(\varphi|S), \hat{p}(y|S)$ and joint distribution $\hat{p}(\varphi;y|S)$, which we are able to do using the tree structure.

The detailed structure of the FFFS algorithm used to compute the filtered feature set $\hat{\Phi}$ requires several subroutines, and the formal algorithmic construction for computing the filter is presented across Algorithms \ref{alg:fffs}, \ref{alg:Recur}, \ref{alg:sf}, and \ref{alg:bin}. The main FFFS algorithm is Algorithm \ref{alg:fffs}, and it calls the subroutines for recursion (Algorithm \ref{alg:Recur}), selecting features (Algorithm \ref{alg:sf}), and partitions (Algorithms \ref{alg:bin}). Formal presentation of these algorithms, as well as detailed descriptions, are given in Appendix \ref{app.fffs}.

%% file: applocalexp.tex
\section{Experimental Validation of Local Explainer}
\label{sec:exp_res}

In this section we empirically evaluate the quality of our local explainer methodology by first showing that accurate sub-type predictions of our PD sub-type clusters (as described in Section \ref{sec:cluster}) can be achieved using black-box methods applied to the data of individuals measured during the screening visit. We then apply our local explainer methodology developed in Section \ref{sec:loc_exp} to explain the predictions given by these black-box models.





Our clusters were derived from longitudinal measurements of the four metrics of disease severity described in Section \ref{s.crit}, measured across the first seven visits in the study over a period of 21 months.  Treating these cluster (and the healthy patients) as our ground truth class labels, we first train black box machine learning models to predict which of these progression sub-types an individual will most likely experience given her screening data. This is meant to model the data available to a physician when she must make treatment decisions for a new patient.  From screening data in the PPMI data set, we included the following 31 features: PTT, Lymphocytes, Hematocrit, Eyes, Psychiatric, Head-Neck-Lymphatic, Musculoskeletal, Sleep Score, Education Years, Geriatric Depression Score, Left Handed, Right Handed, Gender Male, Female Childbearing, Race White, Race Hispanic, Race American Indian, Race Asian, Race Black, Race PI, Anosmia, Hyponosmia, Normosmia, MRI Normal, MRI Abnormal Insignificant, MRI Abnormal Significant, BL/SC UPDRSII, BL/SC UPDRSIII, BL/SC MOCA, BL/SC MSES, and BL/SC Age. Among these 31 features, 20 features are binary variables and 11 features are continuous variables.

For accurate sub-type predictions using this data, in Section \ref{s.mlcluster} we trained three machine learning prediction models: one interpretable model (logistic regression) and two complex black box models (a feed forward ANN and a bagged forest). Our results indicate that the black box models outperform the simpler model, which necessitates the use of a local explainer method for this application to achieve both accurate classification and explainability.

In Section \ref{s.localval} we computed local explanations based on the random forest model predictions (which was the model with the highest accuracy) using our proposed FFFS method with the information filter and a local explainer method.  This is analogous to LIME \citep{ribeiro2016should} which does not contain an information filter. Our results show that given a requirement of high explainer fidelity, the use of the information filter will result in less complex explainer models. All experiments described in this section were run on a laptop computer with a 1.2GHz Intel Core m3-7Y32 processor and MATLAB version R2019a with the machine learning and deep learning tool kits \citep{MATLAB:2010}.





\subsection{Machine Learning Models for Cluster Prediction}\label{s.mlcluster}

We considered three different kinds of machine learning models for the task of predicting the progression cluster: logistic regression, feed forward ANN, and a bagged forest model. The patient data was split into training, validation, and testing sets with $70\%$ of the data used for training, $15\%$ for validation, and $15\%$ for testing. Among 779 patients, 545 patients were selected for training, and 117 patients were selected for validation and testing. 

Since bagged forests and ANNs are sensitive to hyperparamter settings, we used cross-validation to set their respective hyperparamters. Using cross validation and MATLAB's hyperparemeter optimization methods we found that the most effective ANN architecture for our task was with a single hidden layer containing one hundred hidden ReLu units. For the random forest model, we found that an ensemble of 50 bagged trees gave the best results compared to other forest sizes. 

Figures \ref{fig:cm} and \ref{fig:roc} show the performance of the models on the same training, validation, and testing sets. In both figures, the classes 1-4 correspond to Groups 0-3, and class 5 corresponds to healthy patients (which we will also call Group 4). Figure \ref{fig:cm} contains the confusion matrix for each model. The rows of the matrix are the \emph{output class}, which represents the predicted class, and the columns of the matrix are the \emph{target class}, which is the true class. The cells on the diagonal of the matrix count accurate predictions. Each cell in the rightmost column has two values: the top number is the percentage of patients that are correctly predicted to each class, and the bottom number is the percentage of patients that are incorrectly predicted to each class. For each cell on the bottom row, the top number is the percentage of patients that belong to each class and is correctly predicted, and the bottom number is the percentage of patients that are incorrectly predicted. For the rest of cells in the matrix, the number in each cell counts for the number of patients that fall in this observation. The cell at the bottom right corner of each matrix shows the total percentage of patients that were correctly and incorrectly predicted.

As shown in Figures \ref{fig:cm} and \ref{fig:roc}, the logistic regression model under-performs relative to the ANN and bagged forest models. Even though the bagged forest model has a lower prediction rate for Group 0 compared to the ANN, it has equal or higher rates of accurate prediction for the other classes. Additionally, the bagged forest model consistently performed better than the ANN and logistic regression models in our experiments. We concluded from these results that the bagged forest classification model is the most effective for our prediction task, and we chose to consider its predictions when evaluating our local explainer method.

\begin{figure}[h]
	\centering
	\begin{subfigure}[h]{0.32\textwidth}            
		\includegraphics[width=\textwidth]{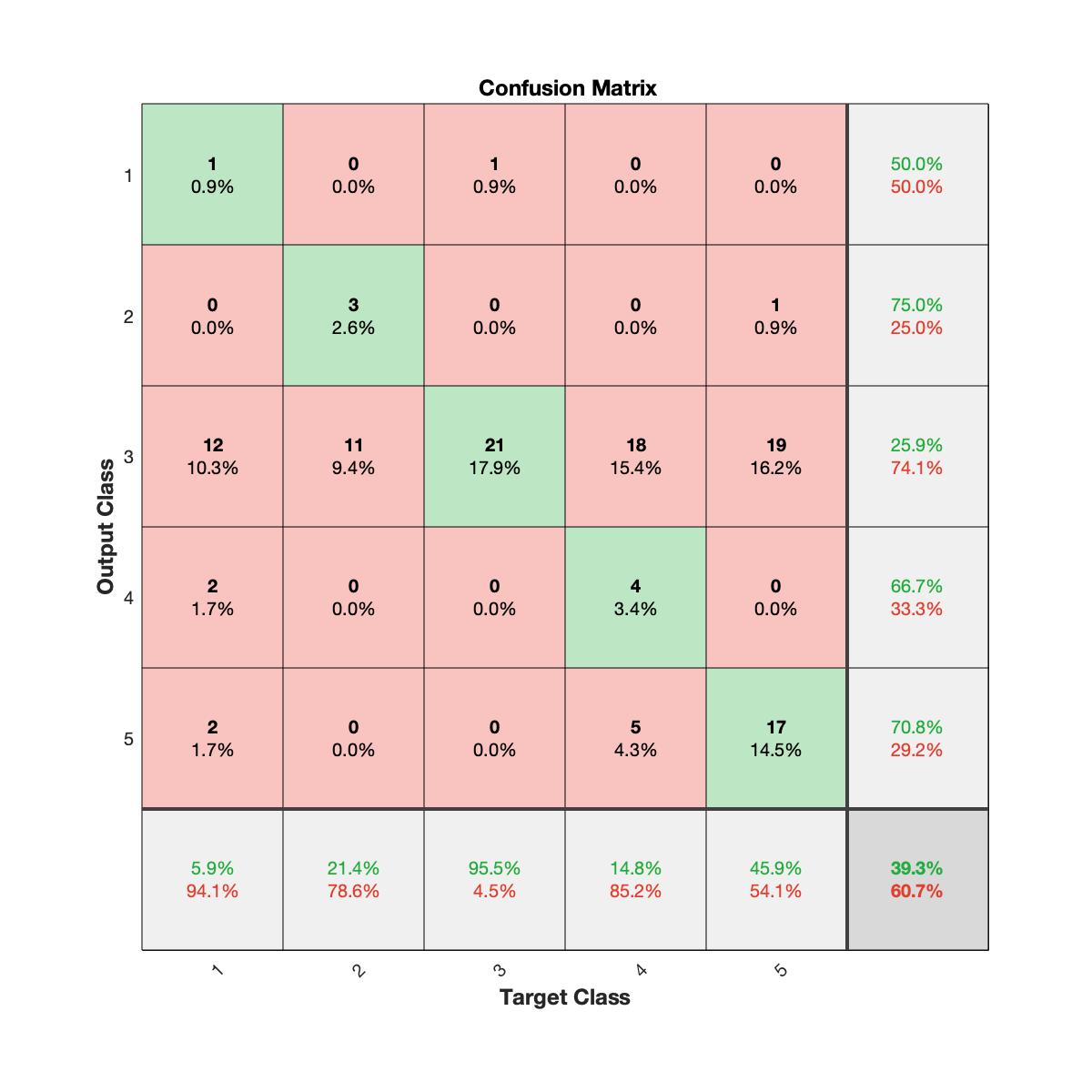}
		\caption{Logistic Regression}
		\label{fig:l1}
	\end{subfigure}
	%
	\begin{subfigure}[h]{0.32\textwidth}
		\centering
		\includegraphics[width=\textwidth]{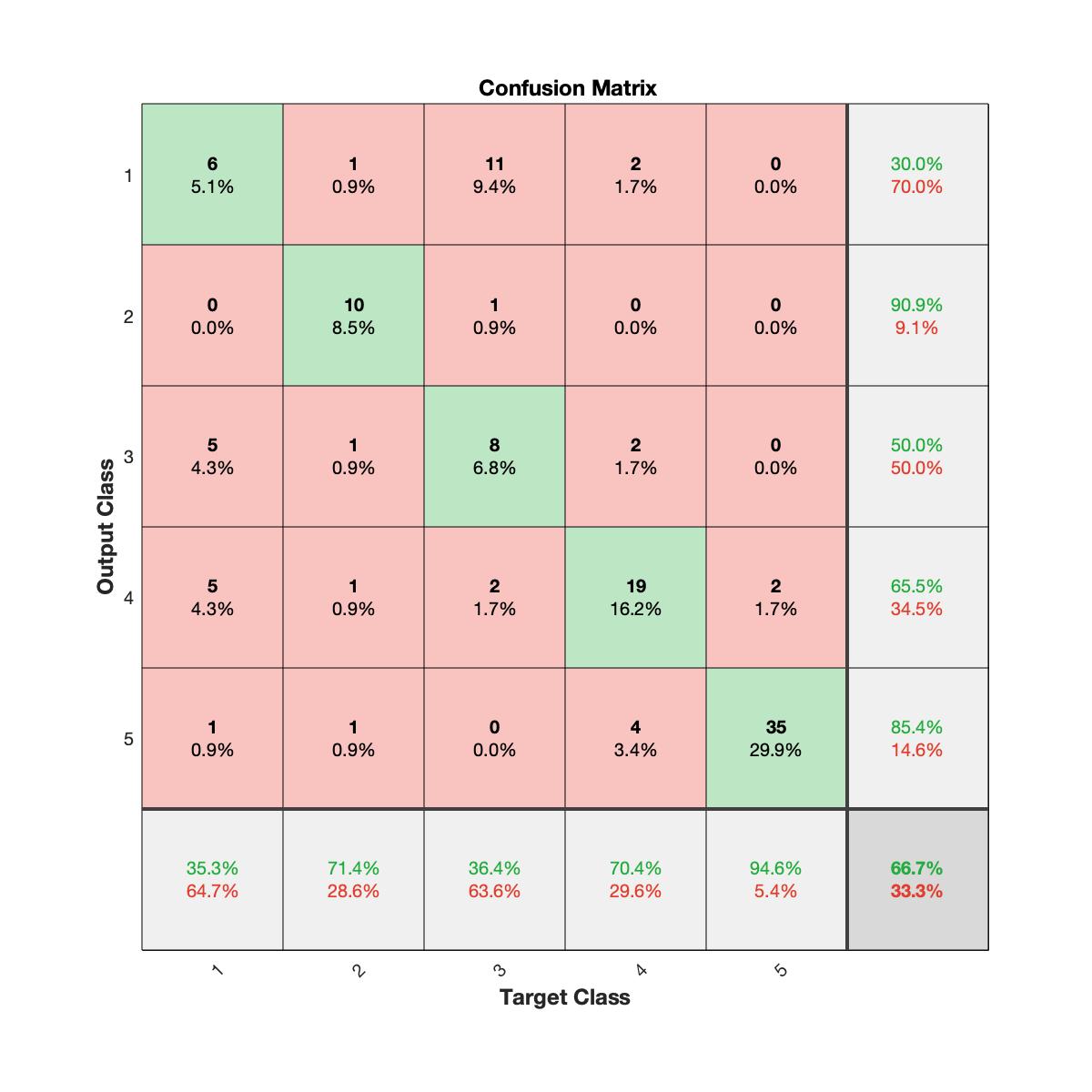}
		\caption{Neural Network}
		\label{fig:n1}
	\end{subfigure}
	\begin{subfigure}[h]{0.32\textwidth}
		\centering
		\includegraphics[width=\textwidth]{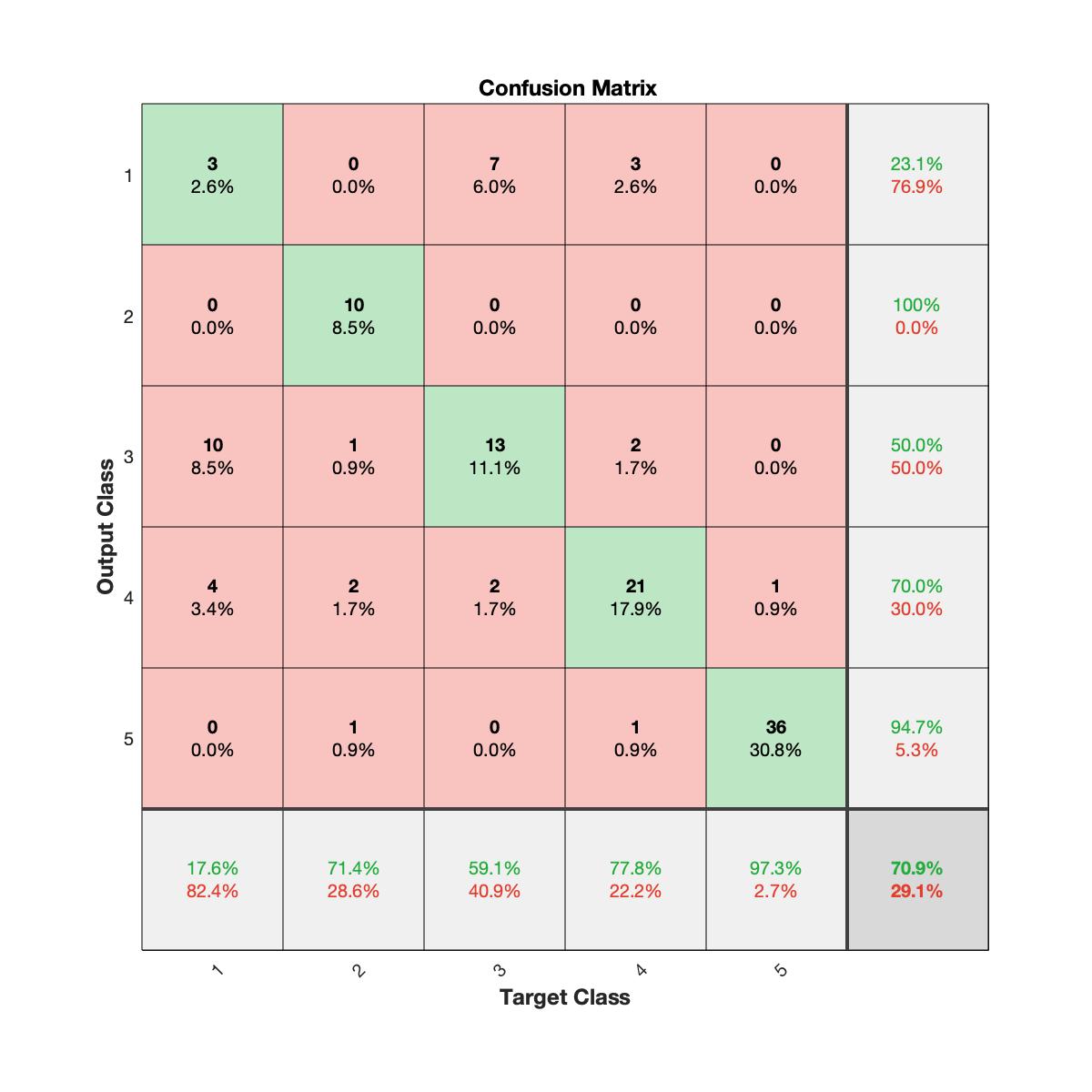}
		\caption{Random Forest}
		\label{fig:r1}
	\end{subfigure}
	
	\caption{Confusion Matrices}\label{fig:cm}
\end{figure}

\begin{figure}
	\centering
	\begin{subfigure}[h]{0.32\textwidth}
		\includegraphics[width=\textwidth]{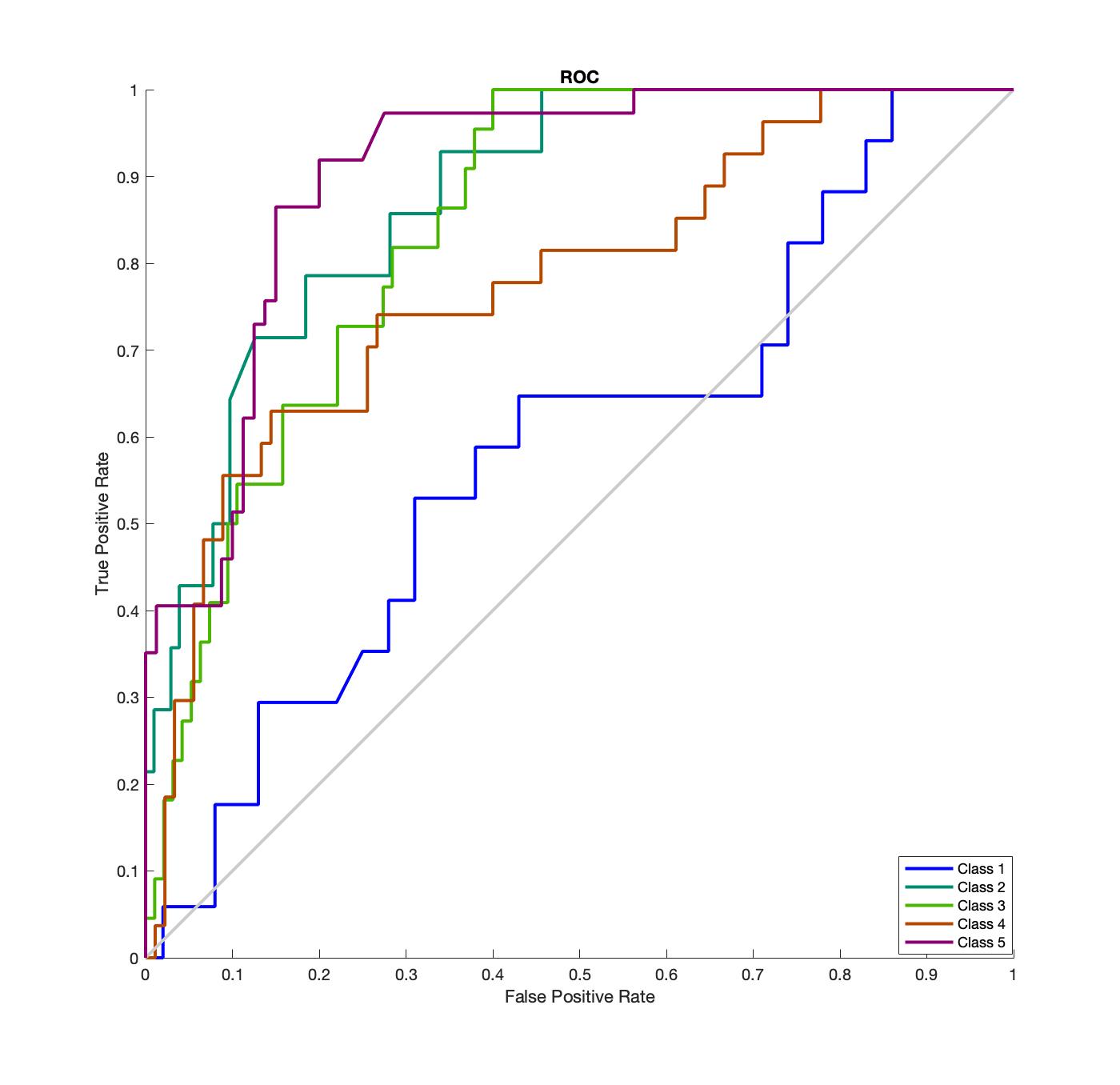}
		\caption{Logistic Regression}
		\label{fig:l2}
	\end{subfigure}
	%
	\begin{subfigure}[h]{0.32\textwidth}
		\centering
		\includegraphics[width=\textwidth]{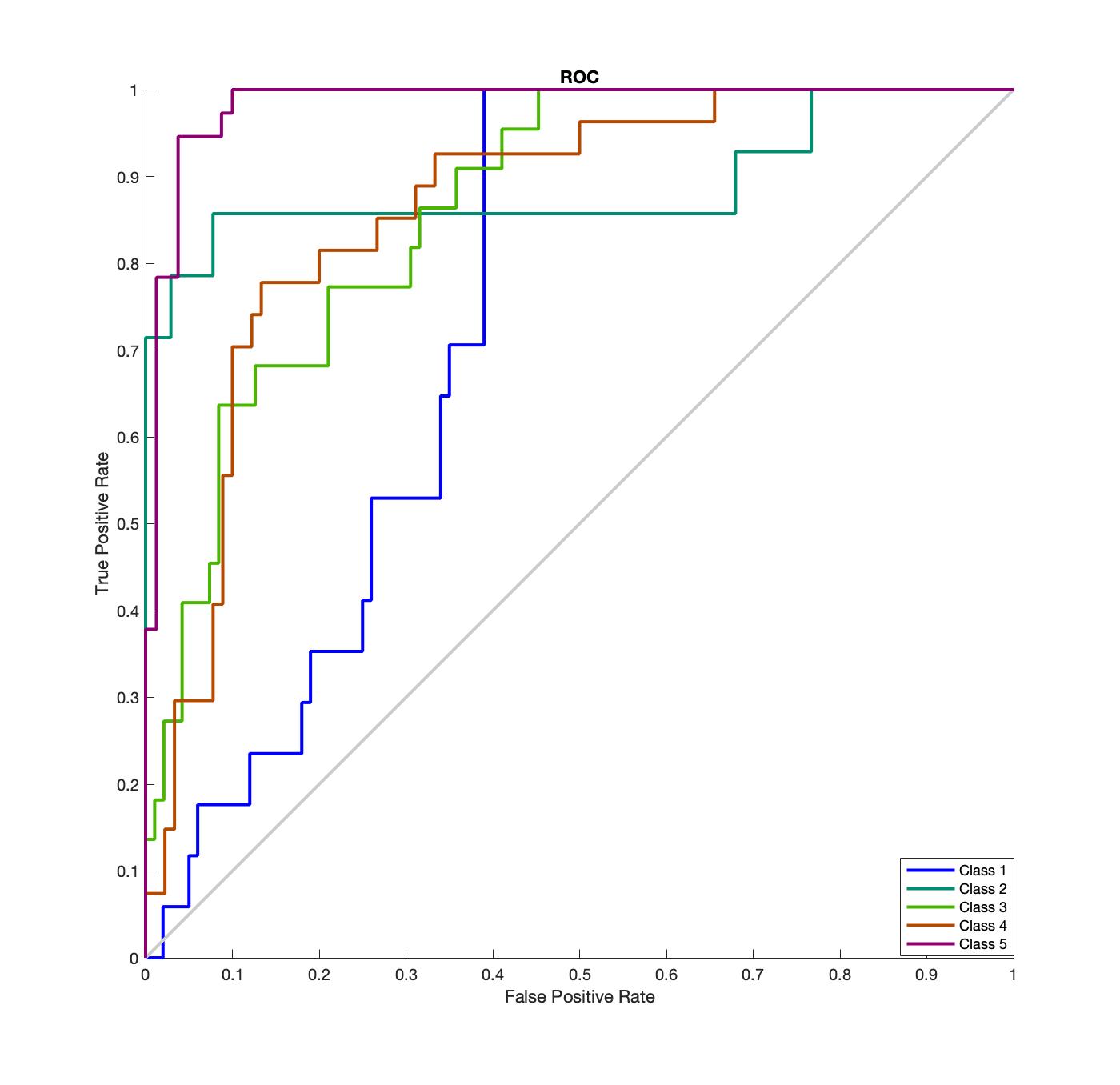}
		\caption{Neural Network}
		\label{fig:n2}
	\end{subfigure}
	\begin{subfigure}[h]{0.32\textwidth}
		\centering
		\includegraphics[width=\textwidth]{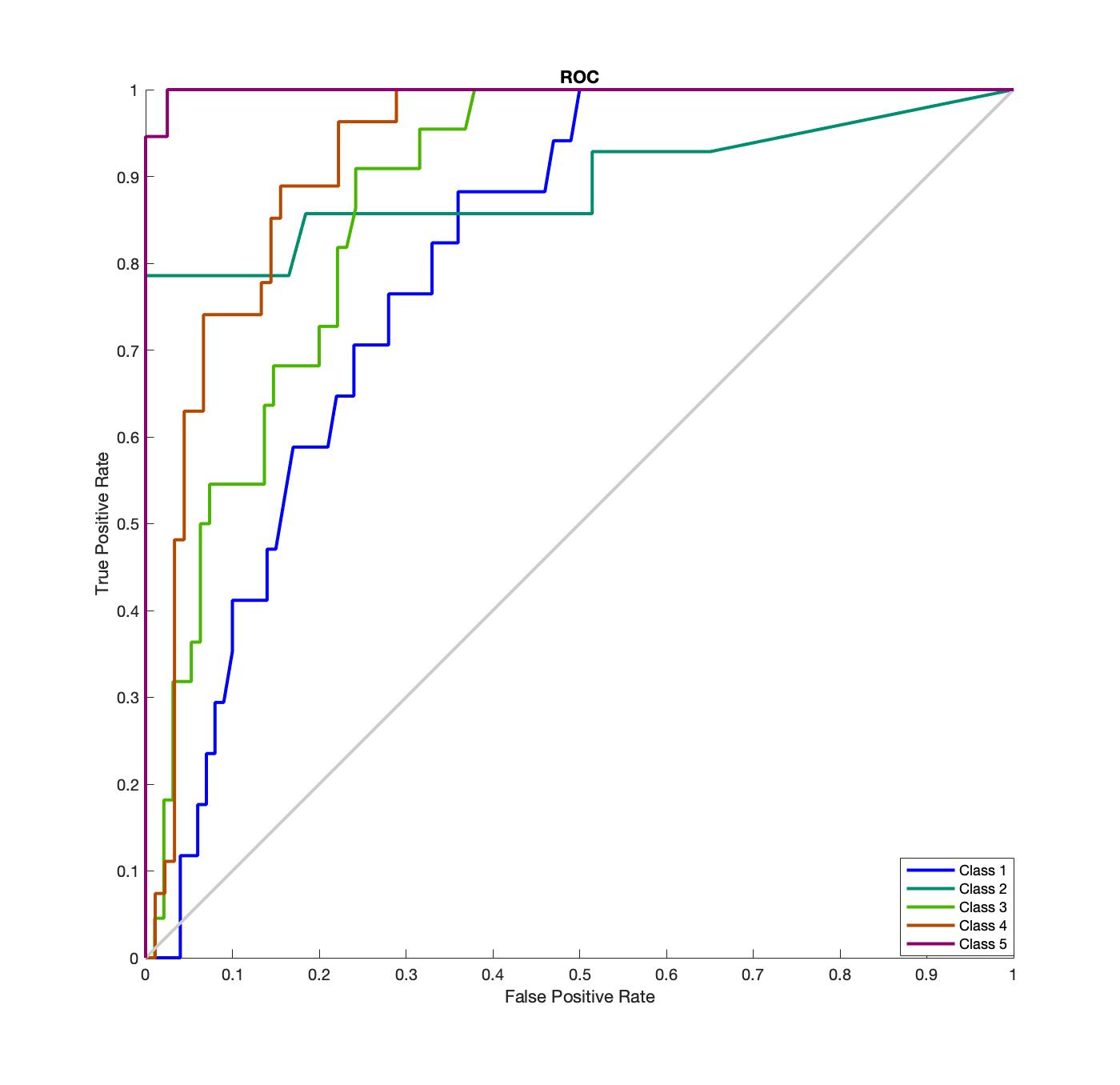}
		\caption{Random Forest}
		\label{fig:r2}
	\end{subfigure}
	\caption{ROC Curves}\label{fig:roc}
\end{figure}

\subsection{Local Explainer Validation}\label{s.localval}

Since the main difference between our local explainer training algorithm and those in the literature is our use of the FFFS information filter, our experiments on the local explainer are focused on validating the effectiveness of using this information filter. We compare the performance of our local explainer training algorithm to a similar algorithm without a filtering step. We then compare the performance of these methods in terms of explainer complexity and fidelity, across different sampling radii and across all patients.

For the sampling parameters of our algorithm, we sampled $N=10,000$ points centered around each patient within a radius $r$ of either 3, 7, 11, or 15. The distance metric for computing this radius was a combination of the $\ell_\infty$ norm for the continuous features and the $\ell_1$ norm for the binary features. The continuous value feature of each of the points was sampled uniformly using standard techniques (\cite{barthe2005probabilistic}).
For binary valued features, we randomly chose at most $r$ binary features and flipped their values. We first randomly generated an integer $k$ between $0$ and $r$, and randomly selected $k$ binary features which we then flipped from their current value (that is, values of 1 were set to 0 and vice versa). 
To compute probability density estimates, we found that the method performed well with histograms with only three bins for continuous features and two bins for binary features. Intuitively three bins allows us to categorize feature values as low, medium, or high relative to their range.


For both training methods, we chose to train decision trees as our the local explainer class because these have been shown to be ergonomically suitable for explaining black box models in healthcare contexts \citep{bastani2018interpreting}. Then we computed the corresponding \emph{fidelity score}, defined as the percentage of data  where the prediction of the decision tree matched the prediction of the random forest model.  We used the number of leaves on the decision tree as a measure of the explainer complexity. 

\begin{figure}[h]
	\centering
	\includegraphics[scale=0.5]{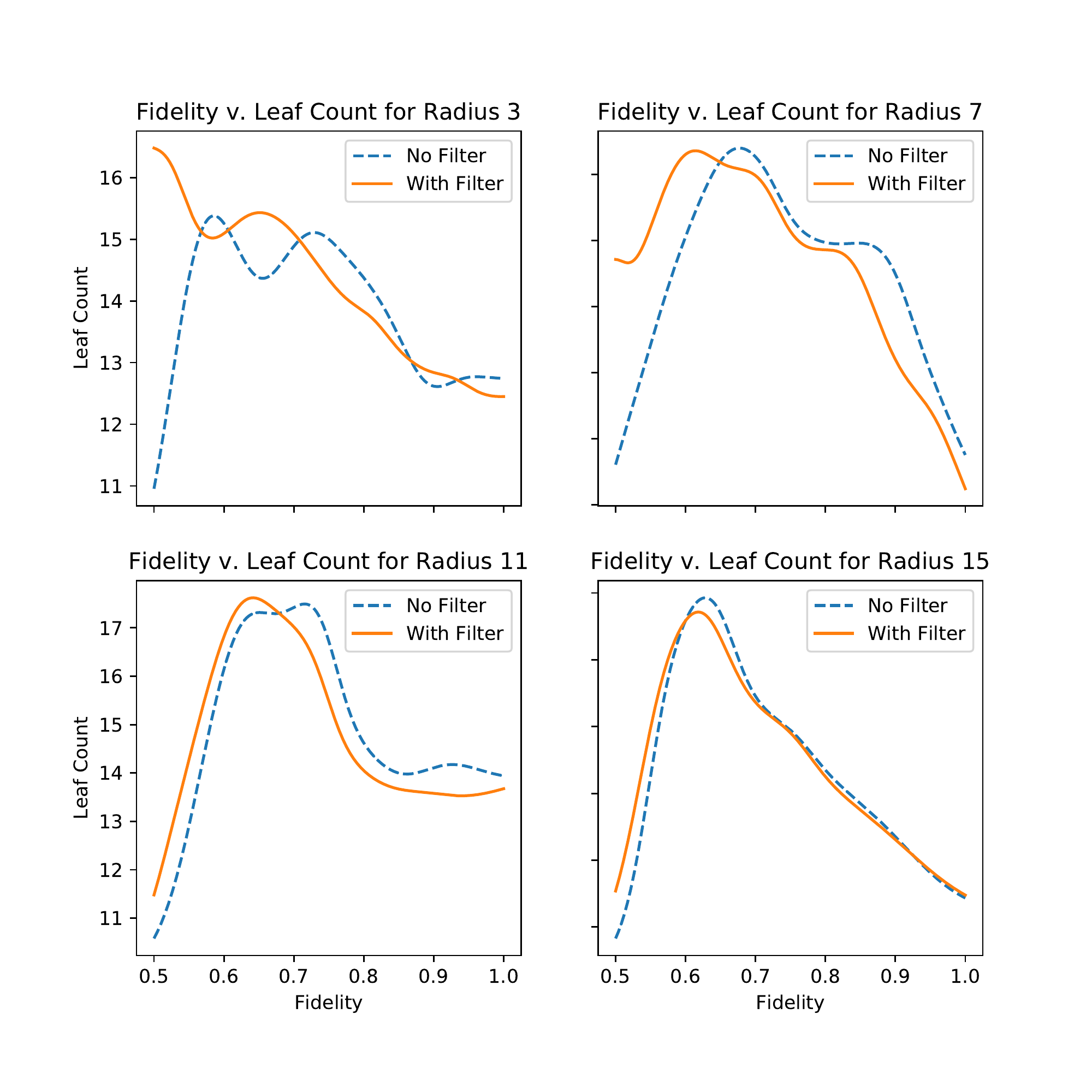}
	\caption{Comparison of local explainer algorithm with the information filter (solid line) and without the the information filter (dashed line) for various different radius settings for the algorithms. The x-axis corresponds to the given fidelity score of the model and the y-axis measures the complexity of the decision tree explainer by the number of leaves. For a small radius ($r=3$) and large radius ($r=15$), the addition of an information filter does not lead to a significant difference in model complexity across all levels of fidelity. However, using the information filter in explainer training for moderate sized radii ($r=7$ and $r=11$) results in less complex models at higher levels of fidelity ($>0.6$).}
	\label{fig:filter_comparison}
\end{figure}

In Figure \ref{fig:filter_comparison}, we compare the explainer complexity and fidelity level of the explainers generated by the two different training methodology across the four different tested sampling radii. Unsurprisingly, when the sampling radius is small (i.e., $r=3$), there is not much advantage to using the information filter in terms of reducing model complexity for a given fidelity level. Since all points are sampled so closely together, the relevant features are easily learned in explainer training. Conversely, when the sampling radius is large ($r=15$), the addition of the information filter only helps slightly. With such a large radius, sampling feature values that are far from the point that is meant to be explained may not give useful information for that prediction. However, when considering the medium radius ranges, for high levels of fidelity, the inclusion of the information filter provides simpler models across the board. In particular, consider the plots corresponding for local explainer radius of $r=7$ and $r=11$ in Figure \ref{fig:filter_comparison}. Note that in both of these figures, when considering high fidelity explainers generated by both methods (fidelity $\geq$ 0.6), the explainers generated by the information filter method are less complex then those generated without the filter. This would indicate that using our information filter, we can obtain high fidelity local explainers that are on average less complex then those generated without this filter. When considering low fidelity explainers, the no filter method creates less complex models then the filter method. This is because our filter method is better equipped to find relevant features even in more complex regions of the black box model, while the no filter method is unable to learn these regions effectively with a fixed sample size. This is significant since this would indicate that our proposed methodology is able to explain a larger portion of the feature space using less complex models while still finding meaningful features for explanations, relative to existing methodologies.

Overall, the plots in Figure \ref{fig:filter_comparison} show that incorporating an information filter into local explainer training can be more effective in extracting relevant features then using regularization, and can generate less complex models with high fidelity. In addition, these results indicate that using an information filter allows for local explainers with information filters to obtain higher fidelity over a larger radius with relatively less complex models. This is particularly significant since less complex models can be me more easily interpreted by domain experts, making it easier for them to translate the clinical significance of the black box model outputs. while larger explanation radii are useful for model validation and generalization of explanations. Moreover, even in complex decision regions generated by the black box model, using an information filter in conjunction with local explainers is better at extracting relevant features for predictions which again can be useful for model validation and providing clinical insights.

%% file: appfffs.tex
\section{FFFS Algorithmic Details}\label{app.fffs}

In this appendix, we present and discuss the FFFS algorithm used in our local explainer method. The main algorithm is presented in Algorithm \ref{alg:fffs}, and the required subroutines are presented in Algorithms \ref{alg:Recur}, \ref{alg:sf}, and \ref{alg:bin}.

Since the main structure of the algorithm requires a recursive tree traversal, Algorithm \ref{alg:fffs} includes a general prepossessing wrapper algorithm that starts the recursion. In this part of the algorithm, the sampled data points are used to compute the empirical densities of their feature values. These densities are approximated using histograms which can vary in the number of bins. For simplicity of presentation, we assume each histogram has the same bin size, but of course this detail can be modified in implementation. The key addition here is the computation of tensor $M$, which tracks the inclusion of each data point's features into their respective histogram bin.

Algorithm \ref{alg:Recur} contains the main recursion of the filter computation, and it outputs the selected features when it terminates. The recursion of Algorithm \ref{alg:Recur} requires a set of selected features $S$, a set of unselected features $U$, the binary tensor $M$, the black box model predictions $Y$, and $\mathcal{L}$, which is a set of partition sets of points in $T(\bar{x})$.
Since no features are selected prior to the first call to Algorithm \ref{alg:Recur}, we initialize the inputs $S =\emptyset$, $U=\Phi$, $Y = f(T(\bar{x}))$ and $\mathcal{L} = T(x_i)$ when it is first called in Algorithm \ref{alg:fffs}.
The recursion terminates and outputs the current set of selected features when either all features are selected or $\mathcal{L}$ becomes empty. If the termination condition is not met,  Algorithm \ref{alg:Recur} calls Algorithm \ref{alg:sf}, which updates $S,U, \text{ and }\mathcal{L}$ using a bin expansion. Then Algorithm \ref{alg:Recur} makes a recursive call with updated inputs and repeat the previous steps.

Algorithm \ref{alg:sf} is used to select one feature from the set of unselected features that maximizes the mutual information $I(\varphi;Y|U)$, and to update $\mathcal{L}$ given the current selected feature. We apply forward selection in Algorithm \ref{alg:sf}. In order to find $\varphi^*=\argmax_{\varphi \in U} I(\varphi;Y|U)$,
we compute $I(\varphi;y|S)$ for each unselected feature $\varphi$. The approximated mutual information $I(\varphi;y|S)$ is computed using the following equation \citep{brown2012conditional}:
\begin{equation*}
I(\varphi;y|S) \approx {\hat{I}(\varphi;y|S)}=\frac{1}{|T(\bar{x})|}\sum_{i=1}^{N}\log\frac{\hat{p}(\varphi;y|S)}{\hat{p}(\varphi|S)\hat{p}(\varphi|S)}.
\end{equation*}
If $I(\varphi^*;y|S)$ is not positive, then we do not select any new features. If no new feature is selected, we terminate the process by setting $U=\emptyset$, which satisfies the termination condition of Algorithm \ref{alg:Recur}, and the feature selection process will be complete. If $I(\varphi^*;y|S)>0$, then we can obtain additional information on the prediction by adding $\varphi^*$ to the set of selected features $S$ and removing it from the set of unselected features $U$. Algorithm \ref{alg:sf} then calls Algorithm \ref{alg:bin} to update $\mathcal{L}$ to $\mathcal{L}'$. Algorithm \ref{alg:bin} is used to partition each set in $\mathcal{L}$ given current selected feature $\varphi^*$. Using the binary tensor $M$, we can collect the set of bins for $\varphi^*$. As an illustrative example of this process, let $B_{\varphi^*}=\{b_1,b_2\}$ and $\mathcal{L}=T(\bar{x})=\{x_1,x_2,....,x_p\}$. Assume $x_i^{\varphi^*}\in b_1$ for $i < 5$ and $x_i^{\varphi^*}\in b_2$ otherwise. Then we can partition the set $\{x_1,x_2,....,x_p\}$ into 2 sets $\ell_1, \ell_2$ s.t. $\ell_1=\{x_1,...,x_4\}$ and $\ell_2=\{x_5,...,x_p\}$. Next we add sets $\ell_1,\ell_2$ to $\mathcal{L'}$. Since $\mathcal{L}$ contains exactly one set, we finish the partition process, and Algorithm \ref{alg:bin} outputs $\mathcal{L'}=\{\{x_1,...,x_4\}, \{x_5,...,x_p\}\}$. 


\begin{algorithm}
	\begin{algorithmic}[1]
		\caption{Fast Forward Feature Selection (FFFS)}
		\label{alg:fffs}
		\Require $T(\bar{x}),\Phi, f$
		\For {$\varphi \in \Phi_c$}
		\State Form histogram with bin set $B_\varphi$ and frequencies $\hat{p}_\varphi$
		\EndFor
		\State set $M\in |B_\varphi| \times |\Phi| \times N$ as a zero tensor
		\For{$x \in T(\bar{x})$}
		\For{$\varphi \in \Phi$}
		\For{$b \in B_\varphi$}
		\If{$x[\varphi] \in b$}\
		\State Set $M[b,\varphi,x] = 1$
		\EndIf
		\EndFor
		\EndFor
		\EndFor
		\State \Return RecursionFFS($\emptyset,\Phi,M,f(T(\bar{x})),T(\bar{x})$)
	\end{algorithmic}
\end{algorithm}

\begin{algorithm}
	\begin{algorithmic}[1]
		\caption{Recursion FFS}
		\label{alg:Recur}
		\Require $S,U, M,Y,\mathcal{L}$
		\If{$U=\emptyset \text{ or } \mathcal{L} = \emptyset$}
		\State \Return $S$
		\Else
		\State $[S',U',\mathcal{L}']= \text{ SelectFeature}(S,U,M,Y,\mathcal{L})$
		\State \Return RecursionFFS$(S',U',M,Y,\mathcal{L}')$
		\EndIf
	\end{algorithmic}
\end{algorithm}

\begin{algorithm}
	\begin{algorithmic}[1]
		\caption{Select Feature}
		\label{alg:sf}
		\Require $S,U, M,Y,\mathcal{L}$
		\State $f^* = \argmax_{f \in U} I(f;Y|U)$
		\If{$I(f^*;Y|U)>0$}
		\State $U = U \setminus f^*$
		\State $S = S \cup f^*$
		\State $\mathcal{L}' $= BinPartition $(M, \mathcal{L},f^*)$
		\State \Return $S,U,\mathcal{L}'$
		\Else
		\State $U=\emptyset$
		\State \Return $S,U,\mathcal{L}$
		\EndIf
	\end{algorithmic}
\end{algorithm}
\begin{algorithm}
	\begin{algorithmic}[1]
		\caption{Bin Partition}
		\label{alg:bin}
		\Require $M,\mathcal{L},f^*$
		\State Use $M$ to find $B_{f^*}$ s.t. $B_{f^*}=\{b_1,b_2,...,b_k\}$ is the set of bins for feature $f^*$
		\State $\mathcal{L'} = \emptyset$
		\For{$\ell \in \mathcal{L} $}
		\State Partition $\ell$ into smaller sets $\{\ell_1,\ell_2,...\ell_k\}$ w.r.t $B_{f^*}$: $\ell_i=\{t \in l : t^{f^*} \in b_i\} \forall i \in \{1,...,k\}$
		\State $\mathcal{L}'= \mathcal{L}' \cup \{l_1,...,l_k\}$
		\EndFor
		\State \Return $\mathcal{L}'$
		
	\end{algorithmic}
\end{algorithm}

\begin{proposition}
	The time complexity of the FFFS algorithm for a fixed maximum discretization bin size is $\mathcal{O}(N|\Phi|)$.
\end{proposition}
\begin{proof}
	Note that the size of the generated points is given by the input parameter $N$, and the set of all features is denoted by $\Phi$. First, since the bin sized is fixed as a constant, and the preprocessing step requires a nested \textbf{for} loop, the total time complexity of the preprocessing is $\mathcal{O}(N|\Phi|)$. The FFFS algorithm operates as a tree traversal, where the depth of the tree at the final stage corresponds to the number of selected features. In each level of the tree, the mutual information of all points is evaluated using Algorithm \ref{alg:sf} and the sets of generated points are partitioned into smaller sets using Algorithm \ref{alg:bin}, which combined require $\mathcal{O}(N)$ operations. Next, since in the worst case, all features contain positive mutual information on the prediction value of the black box model, the maximum possible tree depth is given by $|\Phi|$. Combining these two facts gives the desired result.
\end{proof}

%% file: appfigs.tex
\section{Additional figures}\label{app.fig}

\begin{figure}[htb]
	\centering
	\includegraphics[width=0.5\textwidth]{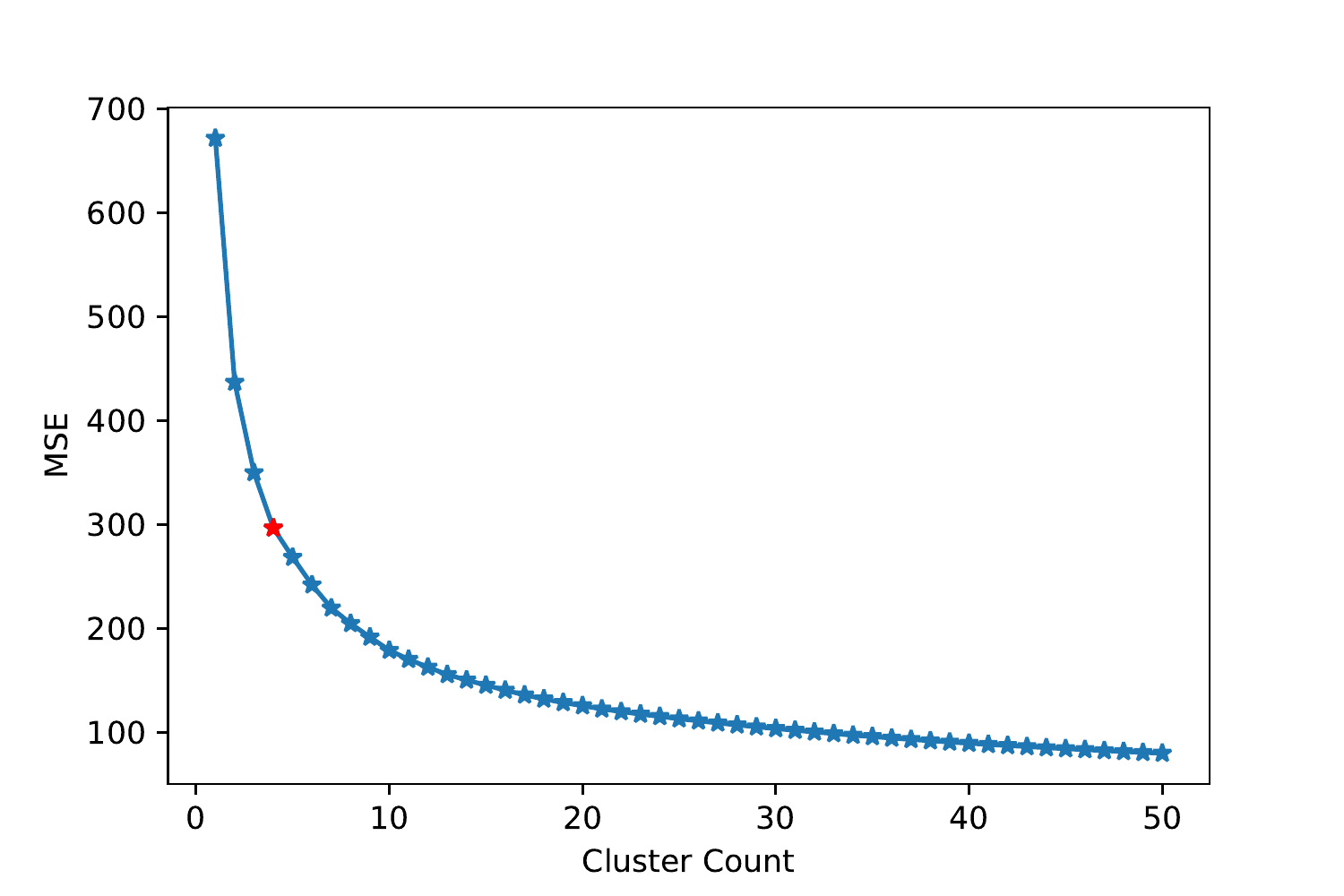}
	\caption{Elbow plot for determining number of clusters to use for $k$-means clustering. Red marked value is located at 4 clusters and roughly corresponds to the bend in the elbow. The $x$-axis describes the total number of clusters used in $k$-means clustering, and the $y$-axis represents the MSE loss associated with the resulting clusters.}
	\label{fig:elbow_plotl}
\end{figure}